\documentclass{article}

\usepackage{microtype}
\usepackage{graphicx}
\usepackage{subcaption}
\usepackage{booktabs} 
\usepackage{amsfonts}
\usepackage{tabularx}

\usepackage{anyfontsize}

\usepackage{bm}
\usepackage{latexsym}
\usepackage{color}
\usepackage{epsfig}
\usepackage{xspace}
\usepackage[english]{babel}
\usepackage{array}
\usepackage{multirow}
\usepackage{makecell}

\newcommand{\ltk}[1]{\textcolor{black}{{#1}}}

\newcommand{\stitle}[1]{\noindent{\bf #1}}

\newcommand{\eat}[1]{}

\newcommand{\stab}{\rule{0pt}{2pt}\\[-2.5ex]}

\newcommand{\ie}{\emph{i.e.,}\xspace}

\newcommand{\lt}{\emph{LT-MILP}}
\newcommand{\rc}{\emph{RC-MILP}}

\newcommand{\ours}{\emph{SC-MILP}}
\usepackage{hyperref}

\usepackage[preprint]{icml2026}

\usepackage{amsmath}
\usepackage{mathtools}
\usepackage{amsthm}

\usepackage[capitalize,noabbrev]{cleveref}

\theoremstyle{plain}
\newtheorem{theorem}{Theorem}[section]

\theoremstyle{definition}
\newtheorem{definition}[theorem]{Definition}

\theoremstyle{remark}

\usepackage[textsize=tiny]{todonotes}

\icmltitlerunning{Submission and Formatting Instructions for ICML 2026}

\begin{document}

\twocolumn[
  \icmltitle{Dynamic Stratified  Contrastive Learning with Upstream Augmentation for  MILP Branching}



  \icmlsetsymbol{equal}{*}
  


  \begin{icmlauthorlist}
    \icmlauthor{Tongkai Lu}{bh}
    \icmlauthor{Shuai Ma}{bh}
    \icmlauthor{Chongyang Tao}{bh}
  \end{icmlauthorlist}

  \icmlaffiliation{bh}{SKLSDE Lab, Beihang University, Beijing, China}

  \icmlcorrespondingauthor{Shuai Ma}{shuaima@buaa.edu.cn}

  \icmlkeywords{Machine Learning, ICML}

  \vskip 0.3in
]



\printAffiliationsAndNotice{}  

\begin{abstract}
  Mixed Integer Linear Programming (MILP) is a fundamental class of NP-hard problems that has garnered significant attention from both academia and industry.
The Branch-and-Bound (B\&B) method is the dominant approach for solving MILPs and the branching plays an important role in B\&B methods. 
Neural-based learning frameworks have recently been developed to enhance branching policies and the efficiency of solving MILPs.
However, these methods still struggle with semantic variation across depths, the scarcity of upstream nodes, and the costly collection of strong branching samples. 
To address these issues, we propose \ours, a Dynamic \underline{\textbf{S}}tratified  \underline{\textbf{C}}ontrastive Training Framework for \underline{\textbf{MILP}} Branching.
It groups branch-and-bound nodes based on their feature distributions and trains a GCNN-based discriminative model to progressively separate nodes across groups, learning finer-grained node representations throughout the tree.
To address data scarcity and imbalance at upstream nodes, we introduce an upstream-augmented MILP derivation procedure that generates both theoretically equivalent and perturbed instances.
\ours~effectively models subtle semantic differences between nodes, significantly enhancing branching accuracy and solving efficiency, particularly for upstream nodes.
Extensive experiments on standard MILP benchmarks demonstrate that our method enhances branching accuracy, reduces solving time, and generalizes effectively to unseen instances.
\end{abstract}

\section{Introduction}

Mixed-Integer Linear Programming (MILP) constitutes a fundamental class of combinatorial optimization problems that integrates discrete and continuous decision variables under linear constraints~\cite{benichou1971experiments}. It not only plays an important role in solving other combinatorial optimization problems, but also serves as a powerful modeling framework for diverse real-world applications, including cloud manufacturing, facility location, task scheduling in Web services, network design, delivery routing, and strategic planning~\cite{zhao2025dual, wang2024fair,guo2025unified,floudas2005mixed, yan2024robust, gao2024lightweight,vielma2015mixed, ashouri2013optimal,bellabdaoui2006mixed,soylu2006synergy,godart2018milp,liu2023mata}. 
The dominant approach for solving MILPs is the Branch-and-Bound (B\&B) algorithm, an exact tree search framework that recursively partitions the feasible space into subproblems~\cite{land1960automatic}.
The algorithm proceeds by building a search tree where each node\footnote{In this paper, ``node'' refers specifically to a B\&B tree node (a MILP subproblem/training instance), while ``vertex'' refers to a node in the MILP's bipartite graph.} in the B\&B tree represents a subproblem defined by additional branching constraints.

In B\&B, variable selection (branching) is the most critical step, as it determines which variable to branch on at each node and thereby dictates the overall structure of the search tree~\cite{fischetti2003local}. Consequently, the quality of branching decisions has a direct impact on both the number of nodes explored and the total solving time.
Traditional variable selection methods rely heavily on expert-crafted rules and are computationally expensive, making them inflexible and time-consuming, especially for large-scale MILPs.
More recently, motivated by the success of neural network methods in other combinatorial optimization domains, researchers have begun exploring their application to variable selection.
\citet{gasse2019exact} map the MILP problem as a bipartite graph and employs Graph Convolutional Neural Networks (GCNNs) to extract features of candidate variables for branching. It adopts an imitation learning approach, treating each node of the B\&B tree as a training sample, specifically targeting strong branching~\cite{applegate1995finding}, which is widely adopted for generating minimal B\&B trees. The study demonstrated notable performance in solving MILPs and set the trend of leveraging neural networks to improve branching strategies~\cite{alvarez2017machine,burges2010ranknet,seyfi2023exact,gupta2020hybrid, zarpellon2021parameterizing, khalil2022mip, lin2022learning, li2025towards, huang2024contrastive, lin2024cambranch, wang2024learning}.
\eat{The Branch-and-Bound (B&B) method iteratively partitions a MILP problem into subproblems and organizes the solving process as a B&B tree, progressively reducing the search space through branching and pruning until the optimal solution is found.}
\eat{, and the search proceeds by solving LP relaxations, applying bounding rules, and pruning infeasible or suboptimal nodes} \eat{While these rules have been refined over decades of algorithmic engineering, they either incur high computational cost or fail to generalize well across problem classes. Motivated by recent advances in machine learning, particularly Graph Neural Networks (GNNs), there has been growing interest in replacing or augmenting these handcrafted rules with learned branching policies that can adapt to instance structure.}

\eat{In B\&B, the branching strategy—how the solver selects the variable to branch on at each node—plays a decisive role in determining the size of the search tree and overall solving time. Strong Branching (SB), for example, often yields the smallest trees by tentatively evaluating each candidate variable through two auxiliary LP solves and selecting the one with the greatest bound improvement. However, the high cost of SB makes it impractical for large-scale instances. Learning-based approaches have emerged as a promising alternative, where a model is trained via imitation learning to mimic SB’s decisions. In this setting, each node of the B\&B tree is treated as an individual training sample, with features derived from the LP relaxation and problem structure, and the SB-chosen variable serving as the ground-truth label. Node states are typically encoded as bipartite graphs linking variable and constraint nodes, allowing GNNs to propagate structural information and produce variable scores. This node-level formulation enables learned policies to replace SB during solving, potentially achieving similar decision quality at much lower cost.}

\eat{Despite the progress made by current neural-based MILP methods, significant challenges persist with the imitation learning
paradigm mentioned above.} 
Despite progress made by current neural-based branching strategies, several critical challenges remain unsolved\eat{significant challenges remain, particularly for branching in upstream nodes near the root of the B\&B tree}. 
First, nodes at different depths correspond to distinct solving phases, resulting in systematic semantic shifts, distributional changes, and variations in decision intent as depth increases. This continuous variation ultimately produces large disparities between upstream and downstream nodes. 
Nevertheless, existing neural-based branching methods~\cite{alvarez2017machine,burges2010ranknet,khalil2016learning,gasse2019exact,gupta2020hybrid,zarpellon2021parameterizing,lin2022learning,seyfi2023exact,huang2024contrastive, lin2024cambranch,li2025towards} overlook such depth-dependent differences by treating all nodes uniformly and averaging their representations, which obscures critical variations and ultimately impairs branching accuracy.
Second, the inherent structural imbalance of the B\&B tree yields far fewer upstream nodes than downstream ones. Consequently, training is dominated by downstream nodes, impairing performance on the critical upstream decisions that largely determine solving efficiency.
Third, acquiring strong branching expert samples for imitation learning is time- and resource-intensive~\cite{lin2024cambranch}, further exacerbating the scarcity of high-quality upstream training data. 
\eat{Third, current models exhibit relatively low top-1 branching accuracy but higher top-3 and top-5 accuracies, yet existing methods do not exploit this characteristic to enhance solving efficiency.} \eat{Together, these factors pose significant challenges for learning accurate branching policies at the root and other upstream nodes.}

To address these challenges, we propose \ours, a Dynamic \underline{\textbf{S}}tratified \underline{\textbf{C}}ontrastive Training Framework for \underline{\textbf{MILP}} Branching. 
Specifically, the framework groups nodes according to their feature distributions, and the GCNN-based discriminative model is trained to regulate the pairwise separation between nodes from different groups, with separation progressively increasing along the group hierarchy. This design enables the model to capture fine-grained semantic variations throughout the B\&B tree. To further mitigate data scarcity and imbalance at upstream nodes, we introduce an upstream-augmented MILP derivation procedure that systematically generates both theoretically equivalent MILPs and perturbed variants from the original instances. By integrating these strategies, SC-MILP effectively captures subtle semantic differences between nodes, thereby enhancing branching accuracy, particularly for upstream nodes.
The main contributions of this work are as follows:

\stab(1) We propose \ours, a novel training framework that initially groups nodes based on their feature distributions and subsequently trains a GCNN-based discriminative model to effectively capture semantic variation in the Branch and Bound tree.

\stab(2) We design an \emph{upstream-augmented MILP derivation} procedure that systematically generates both theoretically equivalent MILPs and perturbed variants from the original instances to address data scarcity and imbalance at upstream nodes without the extra cost of collecting strong branching expert samples.

\stab(3) We present a \emph{dynamic stratified contrastive learning} that contrasts nodes within and across groups, with separation progressively increasing along the group hierarchy. This enables the learning of finer-grained node representations, leading to more informed and effective branching decisions.

\stab(4) Our \ours~significantly improves MILP solving efficiency, outperforming all traditional branching strategies and neural-based methods. Moreover, it enhances branching node prediction accuracy, with particularly pronounced gains at upstream nodes.
\vspace{-1mm}
\section{Related Work}
\subsection{Supervised learning based methods}
Supervised learning methods imitate heuristic branching rules through training predictive models. 
The supervised methods can be broadly categorized into those based on strong branching and those based on alternative heuristics.

\stitle{Strong Branching-based Methods.}
Strong branching~\cite{applegate1995finding} provides highly effective decisions but is computationally prohibitive.
To approximate it, early works~\cite{burges2010ranknet,alvarez2017machine,marcos2014supervised,khalil2016learning} train machine learning models with strong branching labels, but with relatively low branching accuracy.
\citet{khalil2016learning} proposes a GCNN-based framework that represents MILPs as bipartite graphs and predicts branching scores.
Building on this, FILM~\cite{gupta2020hybrid} applies GNNs at upstream nodes and MLPs for others to enhance efficiency, TGAT~\cite{seyfi2023exact} leverages temporal bipartite attention to improve branching accuracy, \citet{chen2024rethinking} employs a second-order folklore GNN to capture complex structures but at a higher computational cost, and MILP-Evolve~\cite{li2025towards} uses LLMs to generate diverse MILPs, enhancing generalization but reducing performance on individual problems.
CAMBranch~\cite{lin2024cambranch} addresses label scarcity via variable-shifting augmentation and contrastive learning between MILPs and AMILPs; however, its reliance on a single augmentation strategy limits data diversity, thereby increasing overfitting risk and hindering the model’s ability to capture complex structural and mathematical patterns.

\stitle{Other Branching Strategies.} 
TreeGate~\cite{zarpellon2021parameterizing} learns reliable pseudo-cost branching via a feature-gated DNN combining both B\&B tree and variable features. Building on this, \citet{pei2023learn} employ an attention-based framework to better exploit variable interactions, MIP-GNN~\cite{khalil2022mip} estimates solution quality using variable deviations, and T-BranT~\cite{lin2022learning} combines a Transformer encoder with graph attention network to capture variable features and global structure. 

\subsection{Reinforcement Learning based Methods}
Reinforcement learning–based methods model branching as a finite-horizon Markov Decision Process (MDP). 
FMSTS~\cite{etheve2020reinforcement} uses deep Q-learning with subtree size as the value function. \citet{scavuzzo2022learning} propose a tree-structured MDP aligned with the search tree, while \citet{zhang2022deep} combine PPO with a GCNN-based value network to guide MCTS for global action optimization. Retro~\cite{parsonson2023reinforcement} converts the search tree into trajectories for efficient learning, and SORREL~\cite{feng2025sorrel} leverages Self-Imitation Learning (SIL) to learn from past high-quality trajectories, improving convergence and exploration.
However, reinforcement learning–based branching remains limited by sparse rewards and credit assignment, performing well only on small MILPs.
Unlike the aforementioned methods that directly learn branching decisions, Symb4CO~\cite{kuang2024rethinking} and GS4CO~\cite{kuang2024towards} leverage reinforcement learning to ``invent'' a good branching algorithm and integrate symbolic learning to improve their reliability.


However, current neural methods often overlook the subtle semantic differences between nodes across depths and data scarcity at the upstream, which limits the model’s ability to learn distinct data representations necessary for accurate branching and effective MILP solving.
Although a few methods, such as CAMBranch~\cite{lin2024cambranch}, acknowledge the high cost of obtaining strong branching samples, they apply uniform augmentation across all nodes, failing to address the imbalance of upstream data. Furthermore, relying on a single augmentation limits their exploitation of MILPs’ structural and mathematical properties, making the model more susceptible to overfitting.
To address these limitations, we propose \ours, a dynamic stratified contrastive training framework, which is specifically designed to capture fine-grained semantic variations across the B\&B tree, thereby increasing branching accuracy and solving efficiency. 
To mitigate upstream data scarcity and imbalance, we introduce an upstream-augmented MILP derivation procedure, which systematically generates both theoretically equivalent MILPs and perturbed variants from the original instances. It provides greater diversity, reduces overfitting, and better exploits structural and mathematical information of MILP instances.

\eat{Supervised learning approaches based on strong branching approximate it with machine learning models by generating datasets using strong branching and then training neural networks to achieve faster computation. Alvarez et al. [64] first designed a machine learning model to simulate strong branching and experimentally demonstrated its effectiveness. Later, Khalil et al. [65] extended this work by modeling the branching process as a ranking problem, selecting the top-ranked variable for branching, and reformulating the score learning task as a classification problem solvable by an SVM classifier. Moreover, inspired by the hybrid strong–pseudocost branching idea, Khalil et al. applied strong branching at the first 500 nodes to collect features and outcomes for model training, and then used the trained ML model for subsequent branching. However, this method still required extensive feature computation at each branching step, resulting in high overhead. To address this, Gasse et al. [66] represented mixed-integer linear programs (MILPs) as bipartite graphs, where nodes correspond to variables and constraints, and edges denote variable–constraint relationships. They employed graph neural networks (GNNs) to learn strong branching scores and make decisions. Gupta et al. [67] proposed a hybrid framework (Figure 8) that combines the expressive power of GNNs with the computational efficiency of multilayer perceptrons (MLPs). In their design, the GNN model with bipartite graph representation is applied only at the root node, while MLPs are used at the remaining nodes, successfully balancing solution quality and computation time.}

\eat{\citet{benichou1971experiments} proposed pseudocost branching (PC), which uses historical branching scores to guide current decisions. Although effective in deeper parts of the search tree, PC suffers from initialization difficulties, since each variable’s pseudocost must be set, causing inefficiency or extra tuning. To address this, \citet{applegate1995finding} introduced strong branching (SB), whose core idea is to simulate branching on all candidate variables, solve the linear relaxations of the resulting subproblems, and then evaluate the impact of each candidate on the objective value. Strong branching produces high-quality decisions but is computationally expensive, making it impractical for large-scale use.

\citet{linderoth1999computational} combined the strengths of pseudocost and strong branching by proposing the hybrid strong and pseudocost (HSB) strategy: strong branching is used in the upper part of the tree, and after a certain depth, pseudocosts are initialized and updated using historical strong branching scores. This approach alleviates the initialization issue of pseudocosts while avoiding the excessive cost of applying strong branching at every node. However, it requires manual specification of the depth at which to switch strategies.

Building on this, \citet{achterberg2005branching} proposed reliability branching (RB). In RB, each variable is associated with a threshold: if a variable is not yet initialized or has been evaluated fewer times than the threshold, strong branching is applied; otherwise, pseudocost branching (initialized by strong branching scores) is used. By dynamically choosing between SB and PC based on pseudocost reliability, RB provides greater flexibility and robustness compared to the fixed-depth HSB strategy. However, selecting appropriate thresholds still requires expert knowledge and extensive experimentation.}
\section{Preliminaries}
\begin{figure*}[t]
 \centering
 \includegraphics[scale=.23]{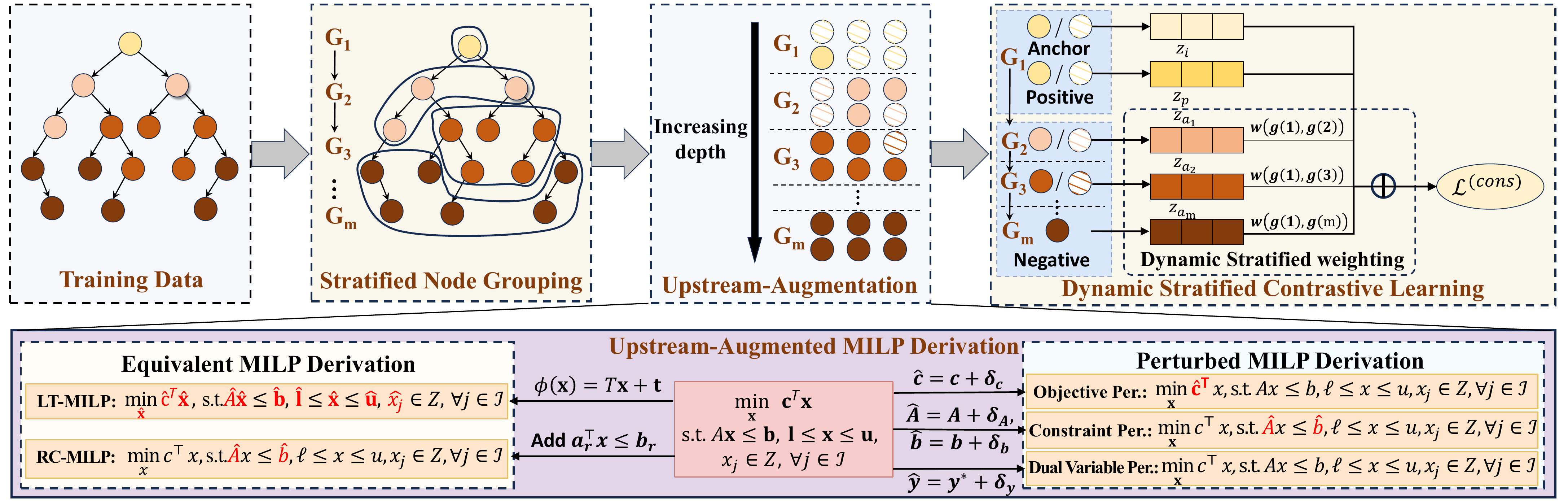}
 \vspace{-2mm}
 \caption{Overview of our \ours. Each node in the B\&B tree is color-coded to reflect feature variations across depths, while nodes marked with slashes denote additional samples generated via upstream-augmented MILP derivation. Nodes are first grouped into feature-driven strata ($G_1, G_2, \dots$). Both equivalent and perturbed MILP derivations are then applied to augment upstream samples. Specifically, \lt~ and \rc~ refer to derivations based on linear transformation and redundant constraint generation, respectively, whereas objective Per., constraint Per., and dual variable Per. indicate perturbations applied to the objective, constraints, and dual variables. Finally, dynamic stratified contrastive learning is performed, where positives are defined within the same stratum and negatives across different strata, with dynamic stratified weighting that progressively increases separation with group depth and is adaptively modulated during training.}
  \vspace{-4mm}
 \label{framework}
\end{figure*}

\stitle{Mixed Integer Linear Programming (MILP).}
We consider the standard form of a Mixed Integer Linear Program (MILP), defined as follows:
\begin{equation}
\setlength\abovedisplayskip{2pt}
\setlength\belowdisplayskip{2pt}
\label{eq:milp}
\begin{aligned}
\min_{{\mathbf{x}}} \ {\mathbf{c}}^{\mathsf{T}} {\mathbf{x}}
\quad \text{s.t.} \quad {A} {\mathbf{x}} \le {\mathbf{b}}, \ {\mathbf{l}} \le {\mathbf{x}} \le {\mathbf{u}}, \ {x}_j \in \mathbb{Z}, \ \forall j \in \mathcal{I},
\end{aligned}
\end{equation}
where \( \mathbf{x} = \{x_1, ..., x_n \} \in \mathbb{R}^{n \times n} \) is the decision variable vector, \( \mathbf{c} \in \mathbb{R}^n \) is the objective coefficient vector, \( A \in \mathbb{R}^{q \times n} \), \( \mathbf{b} \in \mathbb{R}^q \) define the system of linear inequality constraints, \( \mathbf{l} \in (\mathbb{R} \cup \{-\infty\})^n \), \( \mathbf{u} \in (\mathbb{R} \cup \{+\infty\})^n \) are the lower and upper bounds and \( \mathcal{I} \subseteq \{1, 2, \dots, n\} \) is the index set of variables constrained to be integers.

\stitle{Branching for MILPs.}
In a Branch-and-Bound (B\&B) tree $\mathcal{T}$, each node $s$ represents a subproblem with a candidate set $\text{Cand}(s)$.
The goal is to learn a scoring function to evaluate the utility of branching on $x \in \text{Cand}(s)$.
\begin{equation}
 \setlength\abovedisplayskip{2pt}
\setlength\belowdisplayskip{2pt}
    \psi(s, x) : \mathcal{S} \times \text{Cand}(s) \rightarrow \mathbb{R},
\end{equation}
where $\mathcal{S}$ denotes the set of all nodes in $\mathcal{T}$. 
The optimal branching candidate is then chosen as
\begin{equation}
    \setlength\abovedisplayskip{2pt}
\setlength\belowdisplayskip{2pt}
x^\ast(s) = \arg\max_{x \in \text{Cand}(s)} \psi(s, x).
\end{equation}

\stitle{Bipartite Graph Representation of MILP Instances.}
In this paper, we follow a widely used setting in the MILP branching literature~\cite{alvarez2017machine,burges2010ranknet,khalil2016learning,gasse2019exact,gupta2020hybrid,zarpellon2021parameterizing,lin2022learning,seyfi2023exact,lin2024cambranch,li2025towards}, treating each node in the B\&B tree as a distinct training instance for branch learning. 
We model each MILP instance as a bipartite graph $(\mathcal{G}, C, E, V)$, where $\mathcal{G}$ encodes the bipartite structure: an edge $(i, j) \in E$ exists if variable $x_j$ appears in constraint $i$. $C \in \mathbb{R}^{|C| \times d_1}$ and $V \in \mathbb{R}^{|V| \times d_2}$ store constraint and variable vertex features, respectively, and $E \in \mathbb{R}^{|C| \times |V| \times d_3}$ represents edge features. 
The details of these features are summarized in Appendix~\ref{sec:feature}.
\eat{The input $s_i = (\mathcal{G}, C, E, V)$ is then fed into a Graph  Convolutional Neural Network (GCNN)\footnote{We adopt the standard GCNN, which remains widely used due to its strong performance and fast inference speed. Compared with more complex models such as GAT, GCNN enables faster branching decisions, resulting in higher overall efficiency.} for message passing between variable and constraint vertices, producing embeddings for all variables and constraints.
A multilayer perceptron layer (MLP) predicts a score $\hat{s}_j$ for each candidate variable, and the highest-scoring variable $a_i$ is selected during inference.
To optimize the network, cross-entropy is used as a supervised learning loss function
\begin{equation}
\setlength\abovedisplayskip{2pt}
\setlength\belowdisplayskip{2pt}
    \mathcal{L}^{(\text{sup})} = -\frac{1}{N} \sum_{(s_i, a_i^*) \in \mathcal{D}} \log \pi_\theta(a_i^* \mid s_i).
\end{equation}
where $a_i^*$ is the branch decision of strong branching under $s_i$, $\mathcal{D}$ is the training sets, $N = |\mathcal{D}|$ and $\theta$ is the parameters of the GCNN.}

\section{Method}

\subsection{Overview}

In this section, we present \ours, a dynamic stratified contrastive training framework designed for accurate branching and efficient MILP solving, with an overview shown in Fig.~\ref{framework}. \eat{We first introduce stratified node grouping to partition nodes into feature-driven strata, reflecting both structural properties and stage-specific characteristics.}
We first introduce stratified node grouping to partition nodes into feature-driven strata, capturing continuous structural and solving-stage variations for subsequent steps. To address upstream scarcity and imbalance, we then employ upstream-augmented MILP derivation, generating both theoretically equivalent MILPs (via linear transformations of variables and constraints) and perturbed MILPs (lightweight perturbations applied to objectives, constraints, and dual variables), without incurring additional data collection overhead, while simultaneously enriching the training distribution. Finally, we apply dynamic stratified contrastive learning to leverage semantic and feature variations across the tree, thereby improving node representations and branching accuracy. Formally, for each anchor node, nodes within the same group are treated as positives, while nodes from other groups are treated as negatives. We compute pairwise similarities between the anchor node and all other nodes, where dynamic stratified weights are applied to rescale the similarities based on group depth differences. This dynamic stratified contrastive loss progressively preserves intra-group consistency, promotes moderate separation between adjacent strata, and enforces stronger discrimination across distant strata, leading to more informed and effective branching decisions.
\subsection{Stratified Node Grouping}
Considering that nodes at similar depths share comparable solving contexts, we perform \emph{stratified node grouping} that partitions nodes into different strata based on their feature distributions, thereby capturing structural and semantic variations across the B\&B tree and providing fine-grained contextual information for dynamic stratified contrastive learning.

Formally, we partition the node set into $m$ groups, $\{G_1, \dots, G_m\}$, using unsupervised clustering (e.g., K-means\footnote{We adopt K-means as it provides a simple yet effective solution for our task; alternative clustering methods did not yield additional benefits in our experiments.}) based on node features, with the number of groups $m$ determined via the elbow method. This principled grouping establishes a clear stratification that underpins subsequent upstream-augmented MILP derivation and dynamic stratified contrastive learning.

\subsection{Upstream-Augmented MILP Derivation}
Due to the structural imbalance of the B\&B tree, the number of nodes varies significantly across depth, with upstream nodes typically fewer than downstream nodes. This imbalance can negatively impact subsequent dynamic stratified contrastive learning, as the scarcity of upstream nodes limits the model’s ability to capture semantic features at early stages, thereby reducing its discriminative power on critical nodes. After stratified node grouping, the original imbalance becomes more pronounced, appearing as a marked disparity in node counts across groups. A natural solution is to regenerate strong branching expert samples to supplement underrepresented groups. However, this strategy faces two key challenges: (1) generating a sufficient number of samples is computationally expensive, and (2) the newly generated nodes may not faithfully preserve the feature distribution of the original groups, potentially impairing model learning.

To address these issues, we propose an \emph{upstream-augmented MILP derivation} approach that systematically derives new training samples from the original MILP problems rather than relying on additional data collection. Specifically, the equivalence-based derivation applies linear transformations and combinations to produce formally distinct yet equivalent samples in terms of feasible regions and strong branching labels, while the perturbation-based derivation introduces slight modifications to constraints and dual variables to enhance generalization and robustness. Unlike ordinary data augmentation, our method constitutes a strict derivation grounded in the structural properties of the original MILPs with a theoretical foundation and establishes the correspondence between augmented and original bipartite graphs in Appendix~\ref{detalsofaug}, thereby ensuring both fidelity and soundness.

%

\subsubsection{Equivalent MILP Derivation}
\ltk{Intuitively, our goal is to enrich upstream training data without altering the underlying MILP problems. To this end, we leverage the expert knowledge embedded in the original MILPs and introduce two equivalent MILP derivation methods: the linear transformation based derivation (\lt, Definition~1) and the redundant constraint based derivation (\rc, Definition~2).
Each method is rigorously justified through theoretical guarantees (Theorems~\ref{theorem1},~\ref{theorem2}, and~\ref{theorem3}) with formal proofs provided in Appendix~\ref{proofs}.
These ensure a one-to-one correspondence between the feasible regions (including integer variables) of the original and augmented MILPs, as well as between their strong branching outcomes, thereby preserving both fidelity and theoretical rigor.}
\eat{To fully leverage the expert knowledge embedded in the original training MILPs, we propose two equivalent MILP derivation methods, each supported by a theoretical foundation with formal proofs of equivalence. These derivations establish a one-to-one correspondence between the feasible regions (including integer variables) of the original and augmented MILPs, as well as between their strong branching outcomes, thereby ensuring both fidelity and theoretical rigor.}
\vspace{-1mm}
\begin{definition}[Linear transformation based Equivalent MILP Derivation (\lt)] Given an MILP problem in Equ.~\ref{eq:milp}, let \(\phi(\mathbf{x}) = T\mathbf{x} + \mathbf{t}\) be an affine transformation, where $T$ is a special diagonal matrix whose diagonal entries are either 1 or -1 (\(T^{-1} = T\)) and \(\mathbf{t}_{\mathcal{I}}\in\mathbb{Z}^k\). This linear transformation \(\phi(\cdot) \) leads to an equivalent new MILP problem:
\begin{equation}
\label{eq:ltmilp}
\setlength\abovedisplayskip{3pt}
\setlength\belowdisplayskip{3pt}
\begin{aligned}
\min_{\hat{\mathbf{x}}} \ \hat{\mathbf{c}}^{\mathsf{T}} \hat{\mathbf{x}}
\quad \text{s.t.} \quad \hat{A} \hat{\mathbf{x}} \le \hat{\mathbf{b}}, \ \hat{\mathbf{l}} \le \hat{\mathbf{x}} \le \hat{\mathbf{u}}, \ \hat{x}_j \in \mathbb{Z}, \ \forall j \in \mathcal{I}
\end{aligned}
\end{equation}
where \( \hat{\mathbf{c}} = T\mathbf{c} \) is the transformed objective coefficient vector, \( \hat{A} = A T \) is the transformed constraint matrix, \(\hat{\mathbf{b}} = \mathbf{b} + AT\mathbf{t}\) and the left and right bounds \( \hat{\mathbf{l}} = T \mathbf{l} + \mathbf{t} \), \( \hat{\mathbf{u}} = T \mathbf{u} + \mathbf{t}\).
\end{definition}

Next, we present three theorems establishing the equivalence between \lt~ and the original MILP problem, with proofs given in Appendix~\ref{proofs}.
\vspace{-1mm}
\begin{theorem}\label{theorem1}
Let 
\(
\mathcal{P} := \{\, \mathbf{x} \in \mathbb{R}^n \mid A\mathbf{x} \le \mathbf{b},\ \mathbf{l} \le \mathbf{x} \le \mathbf{u} \,\}
\)
be the feasible region of the LP relaxation of Eq.~\ref{eq:milp}, and let
\(
\hat{\mathcal{P}} := \{\, \hat{\mathbf{x}} \in \mathbb{R}^n \mid \hat{A}\,\hat{\mathbf{x}} \le \hat{\mathbf{b}},\ \hat{\mathbf{l}} \le \hat{\mathbf{x}} \le \hat{\mathbf{u}} \,\}
\)
be the feasible region of the LP relaxation of Eq.~\ref{eq:ltmilp}.  
Then: (1) The affine transformation \(\phi(\cdot)\) is a bijection between \(\mathcal{P}\) and \(\hat{\mathcal{P}}\), i.e., the LP feasible domains are in one-to-one correspondence. (2) This bijection \(\phi(\cdot)\) induces a one-to-one correspondence between the sets of optimal solutions of the two LPs.
\end{theorem}

\vspace{-1mm}
\begin{theorem}\label{theorem2}
For any variable \(x_j\) in the original problem of Equ.~\ref{eq:milp}, the strong branching decisions correspond exactly to those of \(\hat{x}_{\pi(j)}\) in the transformed problem of Equ.~\ref{eq:ltmilp}, where \(\pi\) is the variable mapping induced by \(\phi(\cdot)\).
\end{theorem}
\vspace{-1mm}
\ltk{Theorems~\ref{theorem1} and \ref{theorem2} together establish the equivalence between the original and the transformed problems in terms of feasible regions and strong branching outcomes. 
To further broaden the theoretical scope and provide more diverse and complex augmented MILPs, we generalize the linear transformation assumption to a more flexible class of affine mappings in the following theorem.}

\vspace{-1mm}
\begin{theorem}\label{theorem3}
Let the index set of integer variables in Eq.~\ref{eq:milp} be
\(\mathcal{I}=\{1,2,\dots,k\}\), and partition vectors and matrices conformably so that 
\[
\setlength\abovedisplayskip{6pt}
\setlength\belowdisplayskip{6pt}
\mathbf{x} = (\mathbf{x}_{\mathcal{I}}^{T}, \mathbf{x}_{\mathcal{F}}^{T})^{T},
\quad
\mathbf{x}_{\mathcal{I}}\in\mathbb{Z}^k, \;
\mathbf{x}_{\mathcal{F}}\in\mathbb{R}^{n-k}.
\]
Assume the transformation matrix
\[
T = \begin{pmatrix} B & 0 \\[0.3em] F & D \end{pmatrix},
\]
where the blocks and the translation vector \(\mathbf{t} = (\mathbf{t}_{\mathcal{I}}^{T}, \mathbf{t}_{\mathcal{F}}^{T})^{T}\) satisfy: (a) \(B\in\mathbb{R}^{k\times k}\) is a signed permutation matrix (thus each row and column has exactly one entry equal to \(\pm 1\) and all others are zero. Hence \(B\) is invertible and \(B^{-1} = B^{\mathsf{T}}\) with integer entries); (b) \(F\in\mathbb{R}^{(n-k)\times k}\) and \(D\in\mathbb{R}^{(n-k)\times(n-k)}\) with \(D\) invertible; (c) \(\mathbf{t}_{\mathcal{I}}\in\mathbb{Z}^k\).
Then we have that (1) The affine map \(\phi(\mathbf{x})=T\mathbf{x}+\mathbf{t}\) is a bijection between the feasible region of the original MILP and that of the transformed MILP, and it preserves the integrality of the integer components: if \(\mathbf{x}_{\mathcal{I}}\in\mathbb{Z}^k\) then \(\hat{\mathbf{x}}_{\mathcal{I}} = B\mathbf{x}_{\mathcal{I}} + \mathbf{t}_{\mathcal{I}}\in\mathbb{Z}^k\), and conversely. (2) The bijection \(\phi\) induces a one-to-one correspondence between optimal solutions of the two MILPs (and of their LP relaxations).
\end{theorem}
\vspace{-2mm}
\stitle{\textbf{Remark.}} Theorem~\ref{theorem3} shows that the equivalence between the original and transformed MILPs holds under a broader class of linear transformations, far beyond the specific cases considered in Definition~1. 
\ltk{In this paper, we only adopt the simpler transformation forms from Definition~1, as they are easier to implement and already provide sufficient diversity for current benchmarks}.

\vspace{-1mm}
\begin{definition}[Redundant Constraint based Equivalent MILP Derivation (\rc) ]Given an MILP problem in Eq.~\ref{eq:milp}, let \(a_i^\top x \le b_i\) and \(a_j^\top x \le b_j\) be two linearly independent constraints; define a new redundant constraint \(a_r^\top x \le b_r\) with \(a_r = a_i + a_j\), \(b_r = b_i + b_j\). Augmenting the MILP with this constraint yields \(\min_{x} \mathbf{c}^\top x\) s.t. \(\hat{A} {\mathbf{x}} \le \hat{b}, {\mathbf{l}} \le {\mathbf{x}} \le {\mathbf{u}}, x_j \in \mathbb{Z}, \forall j \in \mathcal{I}\), where \(\hat{A} = \begin{bmatrix} A \\ a_r^\top \end{bmatrix}, \hat{b} = \begin{bmatrix} b \\ b_r \end{bmatrix}\).
\end{definition}

As the constraints of \rc~are essentially equivalent to those of the original problem, their LP relaxations share exactly the same feasible region. Consequently, they produce identical strong branching results and optimal solutions. Therefore, \rc~serves as an equivalent derivation of the original MILP.

\subsubsection{Perturbed MILP Derivation}
Equivalent sample derivation ensures strict preservation of feasible regions and strong branching outcomes, providing reliable samples for training. To complement this approach and expose the model to a broader variety of problem structures, we introduce a perturbation-based derivation method that generates non-identical MILPs by applying small perturbations to objective coefficients, constraint coefficients, or dual variables. Although these perturbations may slightly alter the branching outcomes, we retain the original strong branching outcomes during training to encourage the model to learn more resilient representations. Specifically, we consider three types of perturbations: objective perturbation, constraint perturbation, and dual variable perturbation, shown in Definition~3 \&~4.
\vspace{-2mm}
\begin{definition}[Objective and Constraint Perturbation] 
Given an MILP in standard form, we consider two types of perturbations to generate derived MILPs:

\textbf{(1) Objective Perturbation:} The objective vector is perturbed as
\(
\hat{\mathbf{c}} = \mathbf{c} + \delta_c,
\)
where \(\delta_c \in \mathbb{R}^n\) is a small perturbation sampled from a Gaussian distribution. The augmented MILP becomes
\[
\setlength\abovedisplayskip{1pt}
\setlength\belowdisplayskip{1pt}
\min_{{\mathbf{x}}} \ {\hat{\mathbf{c}}}^{\mathsf{T}} {\mathbf{x}}
\quad \text{s.t.} \quad {A} {\mathbf{x}} \le {\mathbf{b}}, \ {\mathbf{l}} \le {\mathbf{x}} \le {\mathbf{u}}, \ {x}_j \in \mathbb{Z}, \ \forall j \in \mathcal{I}
\]

\textbf{(2) Constraint Perturbation:} The constraint system is perturbed as
\(
\hat{A} = A + \delta_A, \quad \hat{\mathbf{b}} = \mathbf{b} + \delta_b,
\)
where \(\delta_A \in \mathbb{R}^{q \times n}\) and \(\delta_b \in \mathbb{R}^q\) are small perturbations sampled from Gaussian distributions. The augmented MILP is
\[
\min_{{\mathbf{x}}} \ {\mathbf{c}}^{\mathsf{T}} {\mathbf{x}}
\quad \text{s.t.} \quad {\hat{A}} {\mathbf{x}} \le {\hat{\mathbf{b}}}, \ {\mathbf{l}} \le {\mathbf{x}} \le {\mathbf{u}}, \ {x}_j \in \mathbb{Z}, \ \forall j \in \mathcal{I}
\]
\end{definition}


\vspace{-2mm}
\begin{definition}[Dual Variable Perturbation] 
For the LP relaxation of the MILP, let \(y^* \in \mathbb{R}^m\) denote the optimal dual variables corresponding to constraints \(A x \leq b\). To enhance robustness to solver sensitivity and numerical instability, we perturb the dual variables by injecting Gaussian noise \(\delta_y\)
\[
\hat{y} = y^* + \delta_y,
\]

\end{definition}

\subsection{Dynamic Stratified Contrastive Learning}
In the Branch-and-Bound (B\&B) tree for solving MILPs, node feature distributions vary significantly across depths. For example, \texttt{sol\_is\_at\_lb} (\texttt{sol\_is\_at\_ub})—indicating whether a variable’s LP solution matches its lower (upper) bound—is typically zero for upstream nodes but mostly one for downstream nodes, where many variables are already fixed. These differences arise because nodes at different depths represent different solving phases: upstream nodes capture global optimization potential, whereas downstream nodes are strongly shaped by prior branching decisions. Learning a single shared representation thus dilutes critical upstream signals, reducing both branching accuracy and generalization.

A natural alternative is supervised contrastive learning (SCL)~\cite{khosla2020supervised}, where nodes from the same group are pulled together and all others are treated as negatives. While SCL captures intra-group similarity, it overlooks the gradual semantic variation across depths. In particular, nodes from adjacent groups—though semantically close—are pushed apart as strongly as distant ones. This abrupt separation neglects the progressive, stratified nature of node semantics, degrading representation quality and limiting branching accuracy, especially for the critical yet underrepresented upstream nodes.

To address these issues, we propose a \emph{dynamic stratified contrastive learning framework}, which explicitly leverages semantic and feature variations across the B\&B tree. Unlike conventional SCL, which treats all negatives uniformly, our framework introduces a dynamic stratified weighting mechanism that weights positive and negative node pairs according to their depth distance: nearby nodes are kept moderately close, while distant nodes are strongly separated. This stratified weighting strategy preserves intra-group consistency, encourages gradual separation across adjacent phases, and enforces stronger discrimination for distant phases, improving the node representations' quality for branching decisions.

We initiate the process with node $i$ (an MILP sample) and its bipartite graph $s_i = (\mathcal{G}_i, C_i, E_i, V_i)$.
Then $s_i$ is fed into a Graph Convolutional Neural Network (GCNN)\footnote{We adopt the standard GCNN, which remains widely used due to its strong performance and fast inference speed. Compared with more complex models such as GAT, GCNN enables faster branching decisions, resulting in higher overall efficiency.} for message passing between variable and constraint vertices, obtaining the variable representations $Z_i = \{z_v \mid v \in V_i\}$ for the entire graph.
For a mini-batch containing multiple problem instances ${s_i}$, we collect all node embeddings from their respective graphs to construct a unified representation pool, enabling cross-instance contrastive learning.
Based on these embeddings, we next introduce both \textit{dynamic stratified contrastive loss} that encourages gradual separation across adjacent groups and the supervised loss for learning branching decisions.

Formally, the dynamic stratified contrastive loss is defined as:
\begin{equation}
\small
\mathcal{L}^{(cons)} \!\!=\! \!\sum_{i} \!-\tfrac{1}{|P(i)|}\!\!\sum_{p\in P(i)}\!\!\!\log\!\frac{\exp\left(\operatorname{sim}\left(Z_{i}, Z_{p}\right)/\tau\right)}{sum_A}\!,
\end{equation}
where $sum_A = \sum_{j\in A(i)}\!\exp\!\left(\!w\!\left(g(i), g(j)\right)\!\cdot\!\operatorname{sim}\!\left(Z_{i}, Z_{j}\right)\!/\tau\right)$, $s_i$ denotes the anchor node, $g(i)$ returns the group index of node $s_i$, $P(i)=\{ p\neq i \mid g(p)=g(i) \}$ is the index set of positive nodes in the same group, $A(i)=\{ a\neq i \}$ denotes all other nodes, and $\operatorname{sim}(\cdot,\cdot)$ is the cosine similarity.

The \textit{dynamic stratified weight} is defined to control the separation strength between nodes from different groups, as follows
\begin{equation}
\setlength{\abovedisplayskip}{3pt}
\setlength{\belowdisplayskip}{3pt}
w(g(i),g(j)) = \sigma(\alpha_{|g(i)-g(j)|}),
\end{equation}
where $\sigma(\cdot)$ is the sigmoid function that bounds the weight within $(0,1)$ for training stability, $\alpha_k$ is a non-decreasing function of the group distance, parameterized as
\begin{equation}
\alpha_k = \sum_{j=1}^{k} \beta_j, \quad \beta_j = \text{softplus}(\theta_j) \ge 0, (\alpha_0=1)
\end{equation}
where $\theta_j$ are learnable parameters and $\text{softplus}(x)=\log(1+e^x)$ ensures positivity for monotonicity and enables backpropagation.
This formulation preserves intra-group consistency and enforces moderate separation between adjacent groups, while dynamically weighting pairwise similarities such that larger weights induce stronger separation and smaller weights reduce their influence.

In addition, to enable the model to learn strong branching decisions, we adopt a commonly used supervised loss based on cross-entropy. Specifically, the variable representations $Z_i$ of the entire graph are fed into a multilayer perceptron layer (MLP) to predict a branching score for each candidate variable $v \in V_i$, and the highest-scoring variable $a_i$ is selected during inference.
The network is then optimized using the cross-entropy loss
\begin{equation}
    \mathcal{L}^{(\text{sup})} = -\frac{1}{N} \sum_{(s_i, a_i^*) \in \mathcal{D}} \log \pi_{\theta_{sup}}(a_i^* \mid s_i).
\end{equation}
where $a_i^*$ is the branch decision of strong branching under $s_i$, $\mathcal{D}$ is the training sets, $N = |\mathcal{D}|$ and $\theta_{sup}$ is the parameters of the GCNN.

The overall training objective combines the original supervised loss $\mathcal{L}^{(sup)}$ with the dynamic stratified contrastive loss:

\begin{equation}
\mathcal{L} = \mathcal{L}^{(sup)} + \lambda \mathcal{L}^{(cons)},
\end{equation}
where $\lambda$ is a learnable parameter as a balancing coefficient, set to 0 during inference.
\eat{which tends to average out critical feature signals—we propose Hierarchical Supervised Contrastive Learning}

\eat{These discrepancies stem from the fact that upstream nodes (near the root) reflect global optimization potential before the feasible region contracts, whereas downstream nodes (near the leaves) are shaped by accumulated local constraints and primarily capture adaptation to the decision path.
Such distributional differences imply that naively learning a unified, shared feature representation across all nodes can lead to an ``averaging'' effect—critical signals are smoothed out, particularly diminishing the model’s branching accuracy and generalization ability at upstream nodes.} 


\eat{Without grouping, standard contrastive learning treats all non-positive samples equally, which can average out important features and reduce the effectiveness of representation learning for upstream nodes.}



\eat{Specifically, the separation strength is parameterized as a learnable weight that depends on the group distance, measured by the average depth gap across groups: embeddings of nodes from adjacent groups are thus only moderately separated, while those from distant groups are enforced to be more strongly distinguished.}

\subsection{Branching with \ours}
\ltk{During MILP solving, \ours~can be invoked at every B\&B branching decision. At upstream nodes near the root, where branching largely determines search time and tree size, we adopt a hybrid strategy: \ours~first predicts scores for all candidate variables to identify the top-k, and strong branching selects the highest-scoring one. Downstream nodes, in contrast, directly use \ours~for branching.
}
\vspace{-4mm}

Since the total depth of the B\&B tree is unknown, we use candidate set size to distinguish upstream from downstream nodes, reflecting both solving stage and branching difficulty: shallow nodes have larger, more ambiguous candidate sets, whereas deeper nodes have smaller, more constrained and informative ones.
The threshold $d_0$ is set as the shallowest depth where a node's candidate set is no more than a fraction $\rho$ of all variables, \ie 
\[
\setlength{\abovedisplayskip}{1pt}
\setlength{\belowdisplayskip}{1pt}
d_0 = \min \left\{ d \,\middle|\, n_{\text{cand}}(d) \le \rho \cdot n_{\text{root}} \right\},
\]  
where $n_{\text{cand}}(d)$ is the number of branching candidates at depth $d$, and variable ratio threshold $\rho \in (0,1)$ is a hyperparameter distinguishing upstream and downstream nodes.
\begin{table*}[t]
  \centering
  \setlength{\tabcolsep}{3pt}
  \caption{Policy evaluation in terms of solving time, number of B\&B nodes, and number of wins over number of solved instances on four combinatorial optimization problems. The best-performing methods in terms of Wins and Time are highlighted in bold, while for Nodes, the best neural-based method is highlighted.}
  \vspace{-2mm}
\aboverulesep=0pt
\belowrulesep=0pt
\resizebox{\textwidth}{!}{
    \begin{tabular}{c|ccc|ccc|ccc|ccc|ccc|ccc}
    \Xhline{4\arrayrulewidth}
    &\multicolumn{9}{c|}{\textbf{Set Covering}}                                   & \multicolumn{9}{c}{\textbf{Combinatorial Auction}} \\
    \cline{2-19}
    \multicolumn{1}{c|}{} & \multicolumn{3}{c|}{Easy} & \multicolumn{3}{c|}{Medium} & \multicolumn{3}{c|}{Hard} & \multicolumn{3}{c|}{Easy} & \multicolumn{3}{c|}{Medium} & \multicolumn{3}{c}{Hard} \\
    Model & \hspace*{-1pt}Wins$\uparrow$ & \hspace*{-4pt}\scalebox{0.9}[1]{Time$\downarrow$}   & \hspace*{-2pt}\scalebox{0.85}[1]{Nodes$\downarrow$ } 
    & \hspace*{2pt}Wins$\uparrow$ & \hspace*{-4pt}Time$\downarrow$ & \hspace*{-3pt}\scalebox{0.93}[1]{Nodes$\downarrow$}
    & \hspace*{-1pt}Wins$\uparrow$ & \hspace*{-6pt}Time$\downarrow$   & \hspace*{-4pt}\scalebox{0.88}[1]{Nodes$\downarrow$ } 
    & \hspace*{-1pt}\scalebox{0.9}[1]{Wins$\uparrow$} & \hspace*{-6pt}\scalebox{0.9}[1]{Time$\downarrow$}   & \hspace*{-3pt}\scalebox{0.9}[1]{Nodes$\downarrow$} 
    & \hspace*{2pt}Wins$\uparrow$ & \hspace*{-4pt}Time$\downarrow$   & \hspace*{-4pt}Nodes$\downarrow$ 
    & \hspace*{-1pt}Wins$\uparrow$ & \hspace*{-5pt}Time$\downarrow$   & \hspace*{-3pt}\scalebox{0.88}[1]{Nodes$\downarrow$ } \\
    \midrule
    FSB~\cite{applegate1995finding} & 0/100 & 17.42 & 16    & 0/100 & 409.32 & 164   & 0/100 & 3600  & n/a   & 0/100 & 4.12  & 6     & 0/100 & 87.45 & 72    & 0/100 & 1821.62 & 401 \\
    RPB~\cite{achterberg2005branching} & 0/100 & 8.91  & 55    & 0/100 & 59.73 & 1734  & 0/100 & 1654.84 & 47352 & 0/100 & 2.73  & 10    & 0/100 & 21.87 & 687   & 0/100 & 137.11 & 5472 \\
    \midrule
    Trees~\cite{alvarez2017machine} & 0/100 & 9.31  & 187   & 0/100 & 92.66 & 4203  & 0/100 & 1800.62 & 45126 & 0/100 & 2.51  & 87    & 0/100 & 23.66 & 980   & 0/100 & 458.35 & 10183 \\
    LMART~\cite{burges2010ranknet} & 0/100 & 7.21  & 171   & 0/100 & 59.86 & 1903  & 0/100 & 1252.01 & 34331 & 0/100 & 1.98  & 75    & 11/100 & 17.42 & 876   & 0/100 & 224.02 & 6149 \\
    svmrank~\cite{khalil2016learning} & 0/100 & 8.07  & 163   & 0/100 & 73.06 & 1937  & 4/100 & 1038.14 & 31089 & 0/100 & 2.34  & 78    & 0/100 & 23.16 & 868   & 0/100 & 376.61 & 6816 \\
    GCNN~\cite{gasse2019exact}  & 11/100 & 6.11  & 171   & 4/100 & 42.44 & 1484  & 0/100 & 1299.99 & 37108 & 0/100 & 1.88  & 72    & 0/100 & 19.31 & 655   & 1/100 & 114.32 & 5231 \\
    FILM~\cite{gupta2020hybrid}  & 3/100 & 6.29  & 165   & 0/100 & 44.32 & 1391  & 0/100 & 1392.42 & 33692 & 11/100 & 1.77  & 73    & 0/100 & 26.04 & 857   & 0/100 & 416.53 & 5310 \\
    TreeGate~\cite{zarpellon2021parameterizing} & 0/100 & 8.32  & 231   & 0/100 & 51.4  & 2410  & 0/100 & 2085.85 & 58536 & 0/100 & 2.35  & 83    & 0/100 & 18.32 & 862   & 0/100 & 176.4 & 5437 \\
    T-BranT~\cite{lin2022learning} & 0/100 & 6.91  & 153   & 0/100 & 43.53 & 1653  & 0/100 & 1154.46 & 34694 & 0/100 & 2.28  & 89    & 0/100 & 19.56 & 723   & 0/100 & 142.73 & 6742 \\
    TGAT~\cite{seyfi2023exact}  & 3/100 & 6.8   & 127   & 1/100 & 46.81 & 1336  & 1/100 & 1174.38 & 29452 & 0/100 & 2.01  & 75    & 3/100 & 22.03 & 694   & 0/100 & 126.49 & 5531 \\
    Symb4CO~\cite{kuang2024rethinking} & 19/100 & 6.09 & 151 & 4/100 & 43.37 & 1438 & 0/100 & 1372.47 & 57315 & 29/100 & 1.69 & 75 & 5/100 & 17.34 & 743 & 4/100 & 108.7 & 5637 \\ 
    GS4CO~\cite{kuang2024towards} & 2/100 & 6.21 & 216 & 3/100 & 42.74 & 1735 & 2/100 & 1147.84 & 63142 & 10/100 & 1.73 & 82 & 1/100 & 18.48 & 746 & 3/100 & 104.39 & 6482 \\ 
    \hspace*{-3pt}\scalebox{0.88}[1]{CAMBranch~\cite{lin2024cambranch}} & 2/100 & 6.33  & 139   & 11/100 & 41.53 & \textbf{1279} & 4/100 & 1104.34 & 31584 & 4/100 & 1.77  & 87    & 3/100 & 17.79 & 683   & 0/100 & 125.94 & 4904 \\
    MILP-Evolve~\cite{li2025towards}& 0/100 & 10.31 & 144   & 3/100 & 46.37 & 1431  & 7/100 & 1024.21 & 30812 & 3/100 & 1.78  & \textbf{64} & 12/100 & 17.5  & 663   & 5/100 & 106.73 & 5316 \\
    \midrule
    Ours  & \textbf{60/100} & \textbf{5.99} & \textbf{117} & \textbf{74/100} & \textbf{37.01} & 1452  & \textbf{82/100} & \textbf{953.24} & \textbf{29375} & \textbf{53/100} & \textbf{1.63} & 68    & \textbf{65/100} & \textbf{16.85} & \textbf{645} & \textbf{87/100} & \textbf{99.81} & \textbf{4721} \\
    \Xhline{4\arrayrulewidth}
    &\multicolumn{9}{c|}{\textbf{Capacitated Facility Location}}                  & \multicolumn{9}{c}{\textbf{Maximum Independent Set}} \\
    \cline{2-19}
    \multicolumn{1}{c|}{} & \multicolumn{3}{c|}{Easy} & \multicolumn{3}{c|}{Medium} & \multicolumn{3}{c|}{Hard} & \multicolumn{3}{c|}{Easy} & \multicolumn{3}{c|}{Medium} & \multicolumn{3}{c}{Hard} \\
    Model & \hspace*{-1pt}Wins$\uparrow$ & \hspace*{-4pt}\scalebox{0.9}[1]{Time$\downarrow$}   & \hspace*{-2pt}\scalebox{0.85}[1]{Nodes$\downarrow$ } 
    & \hspace*{2pt}Wins$\uparrow$ & \hspace*{-4pt}Time$\downarrow$ & \hspace*{-3pt}\scalebox{0.93}[1]{Nodes$\downarrow$}
    & \hspace*{-1pt}Wins$\uparrow$ & \hspace*{-6pt}Time$\downarrow$   & \hspace*{-4pt}\scalebox{0.88}[1]{Nodes$\downarrow$ } 
    & \hspace*{-1pt}\scalebox{0.9}[1]{Wins$\uparrow$} & \hspace*{-6pt}\scalebox{0.9}[1]{Time$\downarrow$}   & \hspace*{-3pt}\scalebox{0.9}[1]{Nodes$\downarrow$} 
    & \hspace*{2pt}Wins$\uparrow$ & \hspace*{-4pt}Time$\downarrow$   & \hspace*{-4pt}Nodes$\downarrow$ 
    & \hspace*{-1pt}Wins$\uparrow$ & \hspace*{-5pt}Time$\downarrow$   & \hspace*{-3pt}\scalebox{0.88}[1]{Nodes$\downarrow$ } \\
    \midrule
    FSB~\cite{applegate1995finding} & 0/100 & 30.49 & 14    & 0/100 & 224.13 & 81    & 0/100 & 748.33 & 52    & 0/100 & 23.57 & 7     & 0/100 & 1581.86 & 38    & 0/100 & 3600  & n/a \\
    RPB~\cite{achterberg2005branching} & 0/100 & 26.37 & 23    & 0/100 & 157.73 & 169   & 0/100 & 645.72 & 105   & 0/100 & 11.33 & 21    & 0/100 & 111.41 & 731   & 0/100 & 2124.75 & 7815 \\
    \midrule
    Trees~\cite{alvarez2017machine} & 0/100 & 28.91 & 133   & 0/100 & 159.88 & 404   & 0/100 & 634.12 & 400   & 0/100 & 10.68 & 73    & 0/100 & 1178.31 & 4643  & 0/100 & 3442.23 & 28210 \\
    LMART~\cite{burges2010ranknet} & 0/100 & 23.36 & 114   & 0/100 & 129.1 & 357   & 0/100 & 520.26 & 345   & 14/100 & 8.31  & 52    & 0/100 & 219.22 & 747   & 0/100 & 3356.55 & 33732 \\
    svmrank~\cite{khalil2016learning} & 0/100 & 23.61 & 116   & 0/100 & 130.88 & 351   & 0/100 & 512.98 & 331   & 0/100 & 13.77 & 45    & 0/100 & 241.83 & \textbf{541} & 0/100 & 3174.23 & 20030 \\
    GCNN~\cite{gasse2019exact}  & 0/100 & 22.15 & 107   & 0/100 & 121.31 & 341   & 0/100 & 563.54 & 345   & 0/100 & 11.44 & 43    & 0/100 & 192.86 & 1837  & 0/100 & 1187.5 & 18668 \\
    FILM~\cite{gupta2020hybrid}  & 0/100 & 21.56 & 104   & 1/100 & 116.81 & 337   & 0/100 & 543.14 & 358   & 0/100 & 10.73 & 47    & 0/100 & 164.57 & 1682  & 0/100 & 3528.71 & 16667 \\
    TreeGate~\cite{zarpellon2021parameterizing} & 0/100 & 21.57 & 126   & 0/100 & 126.8 & 456   & 0/100 & 929.82 & 495   & 0/100 & 11.86 & 56    & 0/100 & 131.3 & 1732  & 0/100 & 3338.97 & 16596 \\
    T-BranT~\cite{lin2022learning} & 0/100 & 18.62 & 142   & 0/100 & 135.22 & 1052  & 0/100 & 638.19 & 1220  & 0/100 & 9.31  & 51    & 0/100 & 113.44 & 1521  & 0/100 & 3338.47 & \textbf{7011} \\
    TGAT~\cite{seyfi2023exact}  & 15/100 & 17.92 & \textbf{96} & 4/100 & 113.12 & \textbf{299}   & 3/100 & 372.71 & 336   & 5/100 & 8.45  & 46    & 17/100 & 96.42 & 1457  & 0/100 & 1201.55 & 18442 \\
    Symb4CO~\cite{kuang2024rethinking}& 0/100 & 19.94 & 121 & 14/100 & 109.66 & 313 & 0/100 & 497.39 & 338 & 0/100 & 10.8 & 49 & 0/100 & 134.72 & 794 & 9/100 & 1095.41 & 17837 \\ 
    GS4CO~\cite{kuang2024towards}& 0/100 & 22.72 & 101 & 0/100 & 119.83 & 327 & 0/100 & 452.83 & 345 & 0/100 & 10.53 & 52 & 0/100 & 131.35 & 852 & 1/100 & 1137.28 & 16341 \\
    \hspace*{-3pt}\scalebox{0.88}[1]{CAMBranch~\cite{lin2024cambranch}} & 9/100 & 18.11 & 97    & 2/100 & 114.53 & 317   & 0/100 & 461.83 & 359   & 2/100 & 9.07  & \textbf{41} & 0/100 & 145.09 & 1648  & 3/100 & 1143.61 & 16473 \\
    MILP-Evolve~\cite{li2025towards}& 2/100 & 20.58 & 103   & 0/100 & 117.65 & 327   & 5/100 & 352.5 & 324   & 2/100 & 8.79  & 48    & 0/100 & 135.7 & 1586  & 0/100 & 1163.47 & 17663 \\
    \midrule
    Ours  & \textbf{74/100} & \textbf{17.14} & 111   & \textbf{79/100} & \textbf{104.38} & 309 & \textbf{92/100} & \textbf{331.3} & \textbf{294} & \textbf{77/100} & \textbf{7.89} & 44    & \textbf{83/100} & \textbf{94.1} & 1385  & \textbf{87/100} & \textbf{964.79} & 14726 \\
    \Xhline{4\arrayrulewidth}
    \end{tabular}}%
  \label{T1}%
  \vspace{-3mm}
\end{table*}%
\section{Experiments}
\subsection{Experimental Settings}

\stitle{Datasets.}
We evaluate our method on four NP-hard benchmarks widely used in prior work~\cite{alvarez2017machine,burges2010ranknet,khalil2016learning,gasse2019exact,gupta2020hybrid,zarpellon2021parameterizing,lin2022learning,seyfi2023exact,lin2024cambranch,li2025towards}, including Set Covering~\cite{balas2009set}, Combinatorial Auction~\cite{leyton2000towards}, Capacitated Facility Location~\cite{cornuejols1991comparison}, and Maximum Independent Set~\cite{bergman2016decision}. All instances are collected from SCIP rollouts with Strong Branching and provided at three difficulty levels (Easy, Medium, Hard), with training sets on the easy level using 100K expert samples and testing on all three levels with 100 instances each. 
We conduct all experiments using SCIP 7.0.1\footnote{Some baseline implementations originally used either SCIP 6.0.1 or 7.0.1. For consistency, we standardize all experiments on SCIP 7.0.1, which ensures reproducibility and fair comparison across methods.} as the backend solver with a one-hour time limit.

\stitle{Baselines.}
We compare our algorithm against a range of B\&B branching methods: (1) conventional methods full strong branching (FSB)~\cite{applegate1995finding} and Reliability Pseudocost
Branching (RPB)~\cite{achterberg2005branching}, (2) machine learning based methods Trees~\cite{alvarez2017machine}, LMART~\cite{burges2010ranknet} and svmrank~\cite{khalil2016learning}, (3) neural based methods GCNN~\cite{gasse2019exact}, FILM~\cite{gupta2020hybrid}, TreeGate~\cite{zarpellon2021parameterizing}, T-BranT~\cite{lin2022learning}, TGAT~\cite{seyfi2023exact}, CAMBranch~\cite{lin2024cambranch}, MILP-Evolve \cite{li2025towards}, and (4) RL based methods Symb4CO~\cite{kuang2024rethinking} and GS4CO~\cite{kuang2024towards}.
\eat{ SORREL~\cite{feng2025sorrel}}

\stitle{Metrics}. Following standard MILP benchmarks, we evaluate solving time (time), the number of nodes in the B\&B tree (nodes), and the number of times each method achieves the best solving time (wins), where lower time and nodes are better, and higher wins are preferred, with solving time being most important. To assess branching decisions, we report top-k accuracy (acc@1, acc@3, acc@5, acc@10), indicating the fraction of samples where the variable selected by the strong branching is among the top-k variables.

\stitle{Parameter Setting.}
\ltk{According to the elbow methods, the group number $m$ for stratified node grouping is set to 4 (Set Covering, Combinatorial Auction, Maximum Independent Set) and 5 (Capacitated Facility Location).
Upstream augmentation applies two equivalent MILP derivations with 35\% probability and three perturbed derivations with 10\% to maintain stability and reliability of upstream samples.
Based on the experimental results in {Exp-4}, the temperature $\tau$ for dynamic stratified contrastive learning is set to 0.08, the hybrid branching ratio $\rho$ is set to 0.8 and $k=5$ since \ours~achieves over 90\% acc@5 (Table~\ref{T2}). 
The training set follows the standard configuration in prior works~\cite{khosla2020supervised,gasse2019exact,gupta2020hybrid, zarpellon2021parameterizing,lin2024cambranch}, and all baselines use their released code with default parameters.}

\begin{table*}[t!]
  \centering
  \setlength{\tabcolsep}{1.5pt}
  \caption{Imitation learning accuracy on the test sets (\%).}
  \vspace{-3mm}
  \aboverulesep=0pt
\belowrulesep=0pt
\resizebox{\textwidth}{!}{
    \begin{tabular}{c|cccc|cccc|cccc|cccc}
    \Xhline{4\arrayrulewidth}
    \multicolumn{1}{c|}{} & \multicolumn{4}{c|}{\textbf{Set Covering}} & \multicolumn{4}{c|}{\textbf{Combinatorial Auction}} & \multicolumn{4}{c|}{\textbf{Capacitated Facility Location}} & \multicolumn{4}{c}{\textbf{Maximum Independent Set}} \\
    \midrule
    Model & acc@1$\uparrow$ & acc@3$\uparrow$ & acc@5$\uparrow$ & acc@10$\uparrow$ & acc@1$\uparrow$ & acc@3$\uparrow$ & acc@5$\uparrow$ & acc@10$\uparrow$ & acc@1$\uparrow$ & acc@3$\uparrow$ & acc@5$\uparrow$ & acc@10$\uparrow$ & acc@1$\uparrow$ & acc@3$\uparrow$ & acc@5$\uparrow$ & acc@10$\uparrow$ \\
    \midrule
    Trees~\cite{alvarez2017machine} & 54.7 & 74.9 & 83.7 & 93.3 & 47.7 & 69.6 & 80.1 & 91.5 & 63.4 & 90.0 & 96.7 & 99.9 & 40.6 & 53.5 & 59.0 & 65.8 \\
    LMART~\cite{burges2010ranknet} & 60.1 & 78.4 & 86.3 & 94.8 & 48.8 & 69.1 & 79.3 & 90.3 & 68.3 & 92.4 & 97.2 & 99.9 & 55.1 & 68.3 & 73.2 & 78.9 \\
    svmrank~\cite{khalil2016learning} & 59.9 & 79.1 & 86.3 & 95.0 & 58.0 & 77.6 & 86.2 & 94.0 & 68.2 & 92.0 & 97.5 & \textbf{100.0} & 55.5 & 69.3 & 74.8 & 81.6 \\
    GCNN~\cite{gasse2019exact}  & 61.8 & 80.9 & 88.9 & 96.3 & 64.1 & 83.4 & 90.8 & 96.8 & 70.4 & 92.9 & 97.9 & \textbf{100.0} & 64.0 & 76.7 & 82.3 & 90.3 \\
    FILM~\cite{gupta2020hybrid}  & 44.1 & 64.8 & 76.0 & 90.2 & 43.6 & 76.6 & 84.2 & 94.7 & 67.5 & 90.6 & 96.6 & 99.9 & 62.3 & 75.1 & 79.1 & 89.7 \\
    TreeGate~\cite{zarpellon2021parameterizing} & 64.1 & 73.6 & 84.3 & 94.6 & 61.4 & 81.5 & 85.9 & 95.2 & 68.5 & 91.5 & 97.0 & 99.9 & 62.8 & 75.8 & 80.6 & 89.8 \\
    TGAT~\cite{seyfi2023exact}  & 68.5 & 79.8 & 89.2 & 97.7 & \textbf{71.4} & 84.4 & 90.5 & 96.9 & 69.5 & 92.4 & 97.4 & \textbf{100.0} & 62.1 & 76.4 & 81.4 & 90.5 \\
    Symb4CO~\cite{kuang2024rethinking} & 57.3 & 73.2 & 85.6 & 94.8 & 60.2 & 79.9 & 83.4 & 93.9 & 66.7 & 89.7 & 96.3 & 99.1 & 58.6 & 74.9 & 79 & 87.4 \\ 
    GS4CO~\cite{kuang2024towards} & 56.8 & 72.7 & 87.2 & 94.5 & 61.5 & 81.4 & 85.7 & 94.3 & 67.1 & 90.9 & 95.7 & 99.4 & 59.5 & 75.6 & 80.3 & 89.2 \\ 
    CAMBranch~\cite{lin2024cambranch} & 60.7 & 78.6 & 87.4 & 96.2 & 63.5 & 82.7 & 90.5 & 96.8 & 68.9 & 91.2 & 96.8 & 99.9 & 63.2 & 76.4 & 81.7 & 90.6 \\
    MILP-Evolve~\cite{li2025towards}& 61.3 & 74.2 & 88.6 & 97.1 & 62.3 & 83.1 & 88.9 & 96.4 & 61.8 & 88.8 & 97.5 & 99.9 & 55.6 & 73.1 & 78.2 & 84.7 \\
    \hline
    Ours  & \textbf{68.9} & \textbf{85.3} & \textbf{92.3} & \textbf{98.2} & 66.4 & \textbf{85.2} & \textbf{91.6} & \textbf{97.4} & \textbf{71.3} & \textbf{93.8} & \textbf{97.6} & \textbf{100.0} & \textbf{65.3} & \textbf{78.6} & \textbf{84.1} & \textbf{91.2} \\
    \Xhline{4\arrayrulewidth}
    \end{tabular}}%
  \label{T2}%
  \vspace{-4mm}
\end{table*}%

\begin{table*}[htbp]
  \centering
  \setlength{\tabcolsep}{1.5pt}
  \caption{Ablation study on the set covering problem.}
  \vspace{-3mm}
  \aboverulesep=0pt
\belowrulesep=0pt
\resizebox{\textwidth}{!}{
    \begin{tabular}{c|cccccc|cc|cc|cc}
    \Xhline{4\arrayrulewidth}
    \multicolumn{1}{c|}{} & \multicolumn{6}{c|}{Accuracies (\%)}               & \multicolumn{2}{c|}{Easy} & \multicolumn{2}{c|}{Medium} & \multicolumn{2}{c}{Hard} \\
\midrule
\multicolumn{1}{c|}{Model} & acc@1$\uparrow$ & acc@1 (top 20\%)$\uparrow$ & acc@3$\uparrow$ & acc@3 (top 20\%)$\uparrow$ & acc@5$\uparrow$ & acc@5 (top 20\%)$\uparrow$ & Time$\downarrow$  & Nodes$\downarrow$ & Time$\downarrow$  & Nodes$\downarrow$ & Time$\downarrow$  & Nodes$\downarrow$ \\
\midrule
    w/o UAMD & 64.2 & 46.7 & 82.6 & 71.9 & 91.1 & 81.0 & 6.04  & 139   & 41.16 & 1479  & 1096.43 & 34461 \\
    w/o EquMD & 64.8 & 47.2 & 83.4 & 72.2 & 91.2 & 81.3 & 6.04 & 136 & 40.85 & 1474 & 1073.76 & 33897\\
    w/o PerMD & 68.1 & 51.1 & 85 & 74.9 & 92.2 & 82.3 & 6.00 & 119 & 37.14 & 1457 & 975.31 & 29841 \\
    \hline
    w/o DSCons & 65.4 & 43.3 & 83.2 & 68.8 & 90.8 & 80.2 & 6.03  & 135   & 41.33 & 1473  & 1137.06 & 33580 \\
    w/o DSWeights & 66.4 & 45.7 & 84.4 & 72.1 & 91 & 81.3 & 6.01 & 126 & 39.82 & 1467 & 1085.46 & 31478 \\ 
    \hline
    w/o HBS & \textbf{68.9} & \textbf{51.3} &\textbf{ \textbf{85.3}} & \textbf{75.3} & \textbf{92.3} & \textbf{82.5} & \textbf{5.95}  & \textbf{107}   & 37.85 & 1467  & 1011.35 & 30258 \\
    \hline
    Ours  & \textbf{68.9} & \textbf{51.3} & \textbf{85.3} & \textbf{75.3} & \textbf{92.3} & \textbf{82.5} & 5.99 & 117 & \textbf{37.01} & \textbf{1452}  & \textbf{953.24} & \textbf{29375} \\
    \Xhline{4\arrayrulewidth}
    \end{tabular}}%
    \vspace{-2ex}
  \label{T3}%
\end{table*}%

\begin{figure*}[htb!]
	\centering
    \subcaptionbox{$m$\label{P2}}[.32\textwidth]{
    \captionsetup{skip=1pt}
    		\includegraphics[width=.30\textwidth]{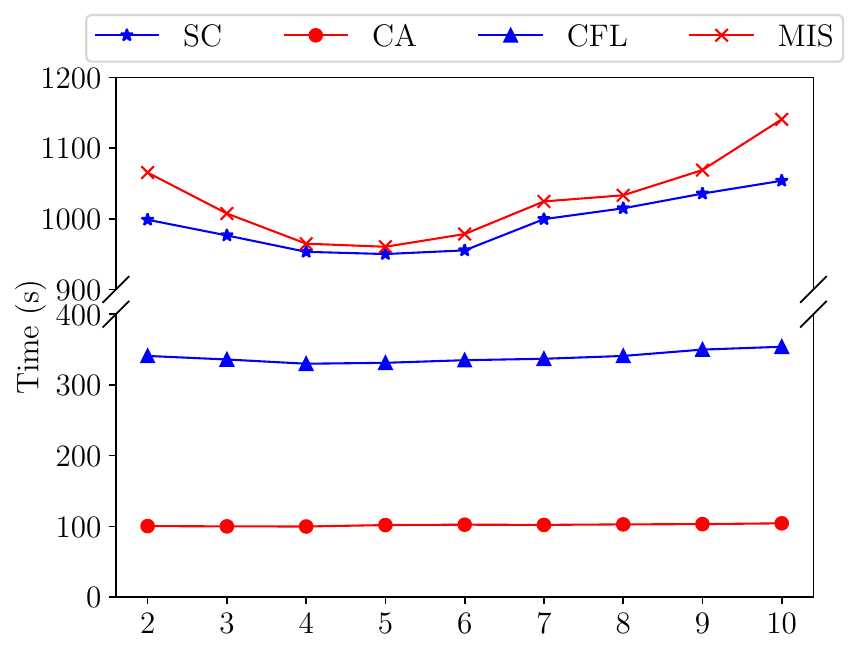}
    	}
	\subcaptionbox{$\tau$\label{P3}}[.30\textwidth]{
    \captionsetup{skip=1pt}
		\includegraphics[width=.30\textwidth]{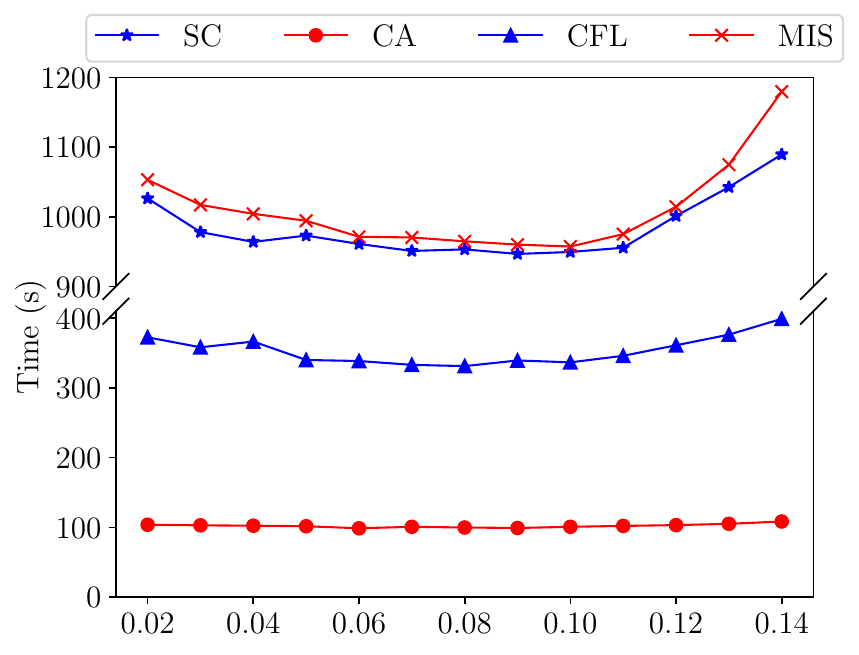}
	}
		\subcaptionbox{$\rho$\label{P4}}[.32\textwidth]{
        \captionsetup{skip=1pt}
		\includegraphics[width=.30\textwidth]{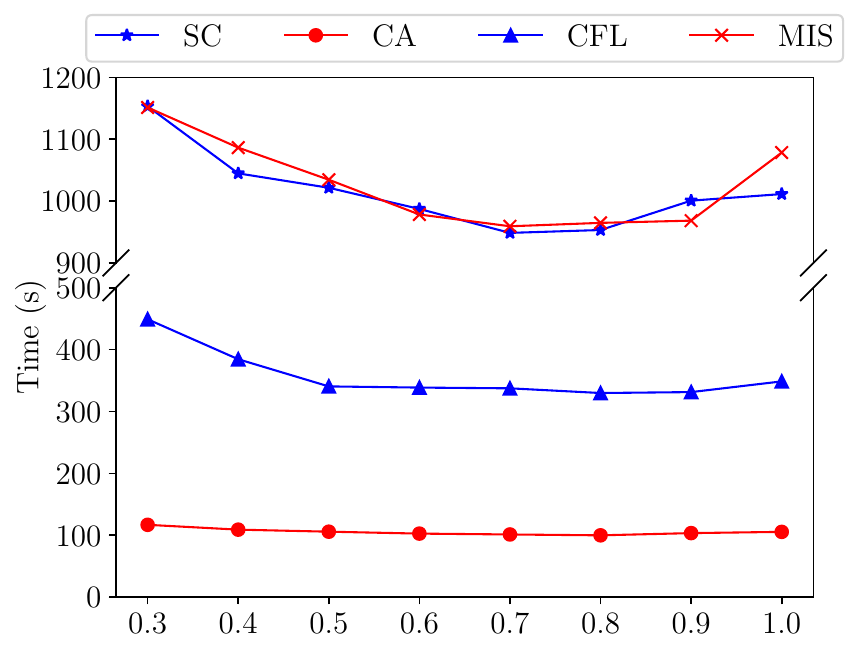}
	}
		\vspace{-3mm}
	\caption{Parameter analysis. In each subplot, SC indicates Set Covering, CA indicates Combinatorial Auction, CFL indicates Capacitated Facility Location, and MIS indicates Maximum Independent Set.}
	\label{P1}
	\vspace{-5mm}
\end{figure*}
\vspace{-3mm}
\subsection{Experimental Results}
\textbf{Exp-1: Overall Comparison with Existing Methods.}
We first compare our method with baseline methods on four different datasets. The results are presented in Table \ref{T1}.
Our \ours~ consistently outperforms all baseline methods in terms of solving time and number of wins, while also generating fewer B\&B nodes in most cases.

Solving time is the most direct indicator of the performance of a branching policy. \ours~achieves the shortest solving time across all datasets, improving efficiency by 12.36\% on average over the current SOTA MILP-Evolve and by 26.48\% over GCNN. On easy, medium, and hard instances, the gains over MILP-Evolve reach 21.25\%, 20.45\%, and 11.25\%, respectively. Notably, almost all neural-based methods surpass the traditional FSB and RPB, highlighting the effectiveness of neural-based branching strategies.

The B\&B tree size directly reflects space usage and generally correlates with solving efficiency.
On average, \ours~ explores 3.28\% fewer nodes than T-BranT, the best-performing baseline in terms of nodes, and 26.48\% fewer than MILP-Evolve, reflecting a substantial improvement in search efficiency and pruning effectiveness.
While occasionally generating more nodes, it still achieves the least solving time, as efficiency depends on both node count and per-node computation. T-BranT’s attention modules slow per-node inference, while our method adopts a hybrid branching strategy, which refines branching decisions with minimal overhead.
Additionally, traditional methods like FSB and RPB may produce fewer nodes but remain less efficient due to costly branching decisions.

The number of wins measures how often a method achieves the shortest solving time among all competitors. Our method achieves the best result in 77.08\% of instances, significantly higher than all baselines. Breaking this down by difficulty, our approach wins on 66.00\% of easy, 75.25\% of medium, and 87.00\% of hard instances. The performance gap becomes larger as the problem difficulty increases, suggesting that our method generalizes better to large and challenging instances beyond the training distribution.

Overall, \ours~ reduces solving time and tree size while achieving the most wins, validating both efficiency and generalization.


\stitle{Exp-2: Performance on Branching Accuracy.}
We evaluate top-$k$ ($k=1,3,5,10$) branching accuracy against neural baselines based on strong branching, excluding FSB and RPB (non-learning) as well as T-BranT (not trained on strong branching), with results in Table~\ref{T3}.
Our method achieves the best performance across all four problems, with average improvements over the current best-performing method TGAT of 0.4\%, 5.1\%, 4.0\%, and 0.7\% for acc@1, acc@3, acc@5, and acc@10, respectively. \eat{While the acc@1 gain is modest, larger gains at acc@3 and acc@5 highlight stronger ranking ability.}
Notably, we aim to balance branching speed and accuracy, which motivates our use of GCNN rather than complex models like TGAT, complemented by optimizing training to improve accuracy. High acc@5 and acc@10 further indicate that top-ranked candidates reliably contain strong branching choices, motivating our hybrid branching strategy.

\stitle{Exp-3: Ablation Study.}
We conduct an ablation study on the set covering problem to evaluate the contributions of Upstream-Augmented MILP Derivation (UAMD) and Dynamic Stratified Contrastive Learning (DSCons). We report acc@1, acc@3, and acc@5 for the top 20\% shallow nodes (acc@1–5 top 20\%) to assess upstream prediction quality.
The results are shown in Table \ref{T3}, where ``w/o UAMD'' and ``w/o DSCons'' denote \ours~without UAMD and DSCons, respectively, ``w/o EquMD'' and `w/o PerMD'' denote \ours~without equivalent and perturbed MILP derivation, respectively,  ``w/o DSWeights'' denotes contrastive learning without dynamic stratified weights and ``w/o HBS'' denotes branching without hybrid branching strategy.

Compared with the ablation variants, the full model achieves the highest accuracy, fastest solving time, and smallest B\&B tree size in most cases, demonstrating the effectiveness and complementarity of each component. Removing UAMD increases solving time by 14.7\% and decreases acc@1 by 4.7\%, with the largest drop caused by omitting equivalent MILP derivation (w/o EquMD) and minimal effect from Perturbed MILP Derivation (w/o PerMD), highlighting the critical role of equivalent MILP derivation in upstream augmentation. Removing DSCons raises solving time by 18.9\% and reduces acc@1 by 3.5\%, while retaining contrastive learning without dynamic stratified weights (w/o DSWeights) partially recovers performance, emphasizing the importance of dynamic stratified weighting. Removing HBS does not affect the intrinsic branching accuracy, as reflected by identical acc@k metrics. On easy instances, removing HBS slightly reduces solving time because acc@1 is already high and the extra computation from top-k strong branching outweighs its benefit. On medium and hard instances, HBS helps reduce solving time and explored nodes, although the gains are modest compared with our two main components.



Notably, accuracy drops for upstream nodes are more pronounced: ``w/o ACons'' and ``w/o UAMD'' cause additional 3.1\% and 0.3\% decreases compared with all nodes, confirming that our method particularly enhances predictions for the most critical upstream portion of the B\&B tree, where accurate branching decisions have the greatest impact on overall solving efficiency.

\stitle{Exp-4: Parameter Analysis.}  
We evaluated the impact of three key hyperparameters: group number $m$, contrastive learning temperature $\tau$, and variable ratio threshold $\rho$. We vary $m$ from 2 to 10, $\tau$ from 0.02 to 0.14, and $\rho$ from 0.3 to 1, while keeping other settings consistent with \textbf{Exp-1}. The range of $\tau$ follows the common setting~\cite{khosla2020supervised}. The results are shown in Fig. \ref{P1}.  
We can find that:  
1) As shown in Fig.~\ref{P2}, model performance first improves with increasing $m$ but then declines as $m$ continues to rise. 
This indicates an optimal trade-off between distinguishing diverse node features with more groups and avoiding over-partitioning, which may disrupt shared patterns among adjacent strata.
Selecting $m$ near this optimal range (4–6) is likely to maximize solving efficiency, in agreement with the results obtained using the elbow methods;
2) The solving efficiency with varying $\tau$ exhibits a similar trend as that of $m$ (Fig. \ref{P3}). Low values of $\tau$ make the model overly sensitive to hard samples, leading to unstable training, while high values weaken the contrastive separability. An intermediate $\tau$ (0.08 to 0.1) thus provides a balance between training stability and feature discriminability, yielding the best overall performance;  
3) The effect of $\rho$ follows a comparable pattern (Fig.~\ref{P4}). 
As $\rho$ increases, more nodes are processed using the hybrid branching strategy, which refines branching decisions but also increases the computational cost at these nodes.
Initially, the gain in accuracy outweighs the added computational cost; however, beyond a certain threshold (0.8 and 0.9), further increases in $\rho$ yield diminishing returns in accuracy relative to the cost. 
This indicates the existence of an optimal trade-off point for $\rho$ that maximizes solving efficiency.
\section{Conclusion}

In this paper, we propose \ours, a Dynamic Stratified Contrastive Training framework for MILP branching.
By grouping nodes according to feature distributions and generating both equivalent and perturbed upstream MILPs, \ours~ trains a discriminative model that encourages moderate separation between adjacent phases and stronger discrimination for distant phases through dynamic stratified contrastive learning. This effectively models fine-grained semantic differences across the B\&B tree, enhances branching accuracy and solving efficiency. 
Experimental results demonstrate that \ours~achieves SOTA branching performance and reduces MILP solving time by an average of 12.36\%. 
In the future, we plan to develop a neural model that generalizes across diverse MILP types, allowing it to handle a wider range of problem instances. 
Beyond the SCIP solver, we also aim to extend \ours~to multiple MILP solvers, further broadening its practical applicability.

\bibliography{paper}
\bibliographystyle{icml2026}

\newpage
\appendix
\onecolumn
\appendix
\section{Branch-and-Bound Framework}
The \emph{Branch-and-Bound} (B\&B) algorithm is the core framework of modern MILP solvers. It maintains a search tree where: (1)
\begin{itemize}
    \item Each node represents a subproblem obtained from~\eqref{eq:milp} by adding branching constraints.
    \item The LP relaxation of the subproblem provides a lower bound on its objective value.
    \item Nodes are \emph{branched} if their LP solution is fractional, \emph{pruned} if their bound is worse than the incumbent, or used to update the incumbent if feasible and better.
\end{itemize}
B\&B is exact: once all nodes are processed or pruned, the best incumbent is the global optimum.

\section{Strong Branching}
A key component of B\&B is the \emph{branching strategy}, i.e., how to choose the variable to branch on.  
\emph{Strong Branching} (SB) is one of the most effective rules: for each fractional candidate variable $x_j$ with LP value $x_j^*$, SB tentatively branches by adding:
\begin{equation}
x_j \le \lfloor x_j^* \rfloor \quad\text{and}\quad x_j \ge \lceil x_j^* \rceil,
\end{equation}
and solves the LP relaxation for each branch.  
Let $z^*$ denote the LP bound of the current node, and $z_j^{\downarrow}$ and $z_j^{\uparrow}$ denote the LP bounds obtained from the down and up branches, respectively.  
The \emph{bound improvements} are defined as:
\begin{equation}
\Delta_j^{\downarrow} = z_j^{\downarrow} - z^*, 
\quad 
\Delta_j^{\uparrow} = z_j^{\uparrow} - z^*,
\label{eq:bound-improvement}
\end{equation}
where $\Delta_j^{\downarrow} > 0$ (or $\Delta_j^{\uparrow} > 0$) indicates that the corresponding branch yields a tighter bound, thereby potentially enabling more pruning in subsequent search.  
The SB score for variable $x_j$ is computed as:
\begin{equation}
s_j = \Delta^{\downarrow}_j + \Delta^{\uparrow}_j,
\end{equation}
and the variable with the highest $s_j$ is selected.  
While SB often yields smaller search trees, it requires solving $2n$ additional LPs for $n$ candidates, making it computationally expensive.

\section{Proofs}\label{proofs}


\begin{proof}[Proof of Theorem~\ref{theorem1}]
(1) \emph{Feasible-Domain Correspondence.}  
Since \(T\) is invertible, \(\phi(\mathbf{x}) = T\mathbf{x} + \mathbf{t}\) is a bijection on \(\mathbb{R}^n\) with inverse \(\phi^{-1}(\hat{\mathbf{x}}) = T^{-1}(\hat{\mathbf{x}} - \mathbf{t})\).  
If \(\mathbf{x} \in \mathcal{P}\), then:
\[
A\mathbf{x} \le \mathbf{b} 
\quad\Longleftrightarrow\quad
A T^{-1} (\hat{\mathbf{x}} - \mathbf{t}) \le \mathbf{b}
\quad\Longleftrightarrow\quad
\hat{A}\,\hat{\mathbf{x}} \le \hat{\mathbf{b}},
\]  
The bound constraints \(\mathbf{l} \le \mathbf{x} \le \mathbf{u}\) transform similarly into \(\hat{\mathbf{l}} \le \hat{\mathbf{x}} \le \hat{\mathbf{u}}\).  
Thus \(\mathbf{x} \in \mathcal{P} \iff \hat{\mathbf{x}} = \phi(\mathbf{x}) \in \hat{\mathcal{P}}\), proving the bijection between feasible domains.

(2) \emph{Optimal-Solution Correspondence.}  
Let the original LP objective be \(f(\mathbf{x}) = \mathbf{c}^\mathsf{T} \mathbf{x}\) and the transformed LP objective be
\[
\hat{f}(\hat{\mathbf{x}}) = \hat{\mathbf{c}}^\mathsf{T} \hat{\mathbf{x}} + \gamma,
\quad\text{where}\quad
\hat{\mathbf{c}} = (T^{-1})^\mathsf{T} \mathbf{c},\quad
\gamma = -\mathbf{c}^\mathsf{T}T^{-1} \mathbf{t}.
\]
Then for any \(\mathbf{x} \in \mathcal{P}\) and \(\hat{\mathbf{x}} = \phi(\mathbf{x})\),
\[
\hat{f}(\hat{\mathbf{x}}) = \mathbf{c}^\mathsf{T} \mathbf{x}.
\]
Suppose \(\mathbf{x}^*\) is optimal for the original LP, but \(\hat{\mathbf{x}}^* = \phi(\mathbf{x}^*)\) is not optimal for the transformed LP.  
Then there exists \(\hat{\mathbf{x}} \in \hat{\mathcal{P}}\) such that \(\hat{f}(\hat{\mathbf{x}}) < \hat{f}(\hat{\mathbf{x}}^*)\).  
Let \(\mathbf{x} = \phi^{-1}(\hat{\mathbf{x}}) \in \mathcal{P}\).  
By the objective equality, \(\mathbf{c}^\mathsf{T} \mathbf{x} < \mathbf{c}^\mathsf{T} \mathbf{x}^*\), contradicting the optimality of \(\mathbf{x}^*\).  
The converse direction follows symmetrically.  
Therefore, the optimal solution sets correspond bijectively under \(\phi\).
\end{proof}


\begin{proof}[Proof of Theorem~\ref{theorem2}]
Let \(f(\mathbf{x}) = \mathbf{c}^\mathsf{T} \mathbf{x}\) be the original objective, and \(\hat{f}(\hat{\mathbf{x}}) = \hat{\mathbf{c}}^\mathsf{T} \hat{\mathbf{x}} + \gamma\) be the transformed objective with \(\hat{\mathbf{c}} = (T^{-1})^\mathsf{T} \mathbf{c}\) and \(\gamma = -\mathbf{c}^\mathsf{T} T^{-1}\mathbf{t}\).

Let $\mathbf{x}^*$ be the optimal solution of the original LP, and let 
$\hat{\mathbf{x}}^* = T \mathbf{x}^* + \mathbf{t}$ be the corresponding 
optimal solution of the transformed LP. Then the optimal objective values are equal:
\[
f(\mathbf{x}^*) = \hat{f}(\hat{\mathbf{x}}^*).
\]

For strong branching on variable \(x_j\), consider the two subproblems formed by adding branching constraints:
\[
\mathcal{P}_{j}^{\downarrow} := \{ \mathbf{x} \in \mathcal{P} \mid x_j \le \lfloor x_j^* \rfloor \}, \quad
\mathcal{P}_{j}^{\uparrow} := \{ \mathbf{x} \in \mathcal{P} \mid x_j \ge \lceil x_j^* \rceil \}.
\]

Under the variable transformation, these constraints correspond to
\[
\hat{x}_{\pi(j)} \le \lfloor \hat{x}_{\pi(j)}^* \rfloor, \quad \hat{x}_{\pi(j)} \ge \lceil \hat{x}_{\pi(j)}^* \rceil,
\]
where \(\pi(j)\) is the index of the variable in \(\hat{\mathbf{x}}\) corresponding to \(x_j\).

Thus, the transformed subproblems are:
\[
\hat{\mathcal{P}}_{\pi(j)}^{\downarrow} := \{ \hat{\mathbf{x}} \in \hat{\mathcal{P}} \mid \hat{x}_{\pi(j)} \le \lfloor \hat{x}_{\pi(j)}^* \rfloor \}, \]
\[\hat{\mathcal{P}}_{\pi(j)}^{\uparrow} := \{ \hat{\mathbf{x}} \in \hat{\mathcal{P}} \mid \hat{x}_{\pi(j)} \ge \lceil \hat{x}_{\pi(j)}^* \rceil \}.
\]

Since the original and transformed feasible regions are bijectively related by
\[
\hat{\mathbf{x}} = T \mathbf{x} + \mathbf{t}, \quad \mathbf{x} = T^{-1}(\hat{\mathbf{x}} - \mathbf{t}),
\]
we have a one-to-one correspondence between \(\mathcal{P}_{j}^\downarrow\) and \(\hat{\mathcal{P}}_{\pi(j)}^\downarrow\), and between \(\mathcal{P}_{j}^\uparrow\) and \(\hat{\mathcal{P}}_{\pi(j)}^\uparrow\).

Moreover, the objectives satisfy
\[
\hat{f}(\hat{\mathbf{x}}) = f(\mathbf{x}), \quad \forall \hat{\mathbf{x}} \in \hat{\mathcal{P}}, \mathbf{x} = T^{-1}(\hat{\mathbf{x}} - \mathbf{t}).
\]

Let the optimal LP values for the subproblems be
\[
z_j^{\downarrow} := \min_{\mathbf{x} \in \mathcal{P}_j^\downarrow} f(\mathbf{x}), \quad
\hat{z}_{\pi(j)}^{\downarrow} := \min_{\hat{\mathbf{x}} \in \hat{\mathcal{P}}_{\pi(j)}^\downarrow} \hat{f}(\hat{\mathbf{x}}).
\]

Since the feasible regions and objectives correspond exactly,
\[
z_j^{\downarrow} = \hat{z}_{\pi(j)}^{\downarrow}.
\]

Similarly,
\[
z_j^{\uparrow} = \hat{z}_{\pi(j)}^{\uparrow}.
\]

Define the strong branching scores (bound improvements) by
\[
\Delta_j^{\downarrow} = z_j^{\downarrow} - z^*, \quad \Delta_j^{\uparrow} = z_j^{\uparrow} - z^*,
\]
and
\[
\hat{\Delta}_{\pi(j)}^{\downarrow} = \hat{z}_{\pi(j)}^{\downarrow} - \hat{z}^*, \quad \hat{\Delta}_{\pi(j)}^{\uparrow} = \hat{z}_{\pi(j)}^{\uparrow} - \hat{z}^*,
\]
where \(z^* = f(\mathbf{x}^*)\) and \(\hat{z}^* = \hat{f}(\hat{\mathbf{x}}^*)\).

Because \(z^* = \hat{z}^*\) and \(z_j^{\downarrow} = \hat{z}_{\pi(j)}^{\downarrow}\), \(z_j^{\uparrow} = \hat{z}_{\pi(j)}^{\uparrow}\), it follows that
\[
\Delta_j^{\downarrow} = \hat{\Delta}_{\pi(j)}^{\downarrow}, \quad \Delta_j^{\uparrow} = \hat{\Delta}_{\pi(j)}^{\uparrow}.
\]

Since strong branching selects the branching variable and direction based on these scores, the decision for \(x_j\) in the original problem corresponds exactly to the decision for \(\hat{x}_{\pi(j)}\) in the transformed problem.

\end{proof}


\begin{proof}[Proof of Theorem~\ref{theorem3}]
\textbf{(1) Bijection of feasible regions.}
Since \(B\) and \(D\) are invertible, \(T\) is invertible with
\[
\setlength\abovedisplayskip{2pt}
\setlength\belowdisplayskip{2pt}
T^{-1} =
\begin{pmatrix}
B^{-1} & 0 \\[0.3em]
- D^{-1} F B^{-1} & D^{-1}
\end{pmatrix}.
\]
The affine transformation \(\phi(\mathbf{x}) = T\mathbf{x}+\mathbf{t}\), therefore, has the inverse
\[
\setlength\abovedisplayskip{2pt}
\setlength\belowdisplayskip{2pt}
\phi^{-1}(\hat{\mathbf{x}}) = T^{-1}(\hat{\mathbf{x}} - \mathbf{t}),
\]
so \(\phi\) is bijective on \(\mathbb{R}^n\).
By construction, the transformed MILP is obtained from the original one by substituting \(\mathbf{x} = T^{-1}(\hat{\mathbf{x}} - \mathbf{t})\) into all constraints.  
Hence \(\mathbf{x}\) is feasible for the original MILP if and only if \(\hat{\mathbf{x}} = \phi(\mathbf{x})\) is feasible for the transformed MILP.

\textbf{(2) Preservation of integrality.}
For the integer block, we have
\[
\setlength\abovedisplayskip{2pt}
\setlength\belowdisplayskip{2pt}
\hat{\mathbf{x}}_{\mathcal{I}} = B\mathbf{x}_{\mathcal{I}} + \mathbf{t}_{\mathcal{I}}.
\]
Because \(B\) is a signed permutation matrix, \(B\mathbf{z} \in \mathbb{Z}^k\) for all \(\mathbf{z} \in \mathbb{Z}^k\).  
With \(\mathbf{t}_{\mathcal{I}}\in\mathbb{Z}^k\), it follows that \(\mathbf{x}_{\mathcal{I}} \in \mathbb{Z}^k \implies \hat{\mathbf{x}}_{\mathcal{I}} \in \mathbb{Z}^k\).  
Conversely, if \(\hat{\mathbf{x}}_{\mathcal{I}}\in\mathbb{Z}^k\) then
\[
\setlength\abovedisplayskip{2pt}
\setlength\belowdisplayskip{2pt}
\mathbf{x}_{\mathcal{I}} = B^{-1}(\hat{\mathbf{x}}_{\mathcal{I}} - \mathbf{t}_{\mathcal{I}}) \in \mathbb{Z}^k
\]
because \(B^{-1}\) is also a signed permutation matrix.

\textbf{(3) Correspondence of optimal solutions.}
Let the original MILP objective be \(f(\mathbf{x}) = \mathbf{c}^\mathsf{T} \mathbf{x}\).  
Define the transformed objective as
\[
\setlength\abovedisplayskip{2pt}
\setlength\belowdisplayskip{2pt}
\hat{f}(\hat{\mathbf{x}}) = \hat{\mathbf{c}}^\mathsf{T} \hat{\mathbf{x}} + \gamma,
\quad
\hat{\mathbf{c}} = (T^{-1})^\mathsf{T} \mathbf{c},
\quad
\gamma = -\mathbf{c}^\mathsf{T} T^{-1} \mathbf{t}.
\]
Then for all \(\mathbf{x}\) and \(\hat{\mathbf{x}} = \phi(\mathbf{x})\) we have \(\hat{f}(\hat{\mathbf{x}}) = f(\mathbf{x})\).  
If \(\mathbf{x}^*\) is optimal for the original MILP, \(\hat{\mathbf{x}}^* = \phi(\mathbf{x}^*)\) is feasible for the transformed MILP with the same objective value.  
If \(\hat{\mathbf{x}}^*\) were not optimal, there would exist \(\hat{\mathbf{y}}\) feasible with \(\hat{f}(\hat{\mathbf{y}}) < \hat{f}(\hat{\mathbf{x}}^*)\), implying that \(\mathbf{y} = \phi^{-1}(\hat{\mathbf{y}})\) is feasible for the original MILP with \(f(\mathbf{y}) < f(\mathbf{x}^*)\), contradicting optimality.  
The reverse direction is identical.  
Thus, optimal solutions correspond one-to-one.
\end{proof}

\section{Details of Upstream-Augmentation}\label{detalsofaug}
\subsection{Bipartite Graph Vertex Features}\label{sec:feature}
Following \citet{gasse2019exact}, we model the MILP corresponding to each node in the Branch-and-Bound (B\&B) tree with
a bipartite graph, denoted by $(\mathcal{G}, C, E, V)$, where the details of these features are shown in Table~\ref{tab:milp_features}.

\begin{table*}[ht]
\centering
\caption{An overview of the features for constraints, edges, and variables in the bipartite graph
$s_i = (\mathcal{G}, C, E, V)$ following \citet{gasse2019exact}. C = constraint vertex, E = edge, V = variable vertex.}
\vspace{-3mm}
\label{tab:milp_features}

\resizebox{0.95\textwidth}{!}{
\begin{tabular}{ccc}
\hline
\textbf{Type} & \textbf{Feature} & \textbf{Description} \\
\hline
\multirow{5}{*}{C} 
    & obj\_cos\_sim & Cosine similarity between constraint and objective coefficients. \\
    & bias          & Normalized right-hand side (RHS) value. \\
    & is\_tight     & Indicator of whether constraint is tight in LP solution. \\
    & dualsol\_val   & Normalized dual value of the constraint. \\
    & age           & LP age since last improvement on the current vertex. \\
\hline
\multirow{1}{*}{E}
    & coef          & Normalized coefficient $a_{ij}$ linking variable $x_j$ to constraint $c_i$. \\
\hline
\multirow{9}{*}{V}
    & type          & One-hot encoding for variable type (binary variables, integer variables, implicit integer variables, and continuous variables). \\
    & coef          & Normalized objective coefficient $c_j$. \\
    & has\_lb / \_ub & Indicator whether variable has lower/upper bounds. \\
    & sol\_is\_at\_lb / \_ub & Indicator whether LP solution is at lower/upper bound. \\
    & sol\_frac      & Fractionality of LP solution value. \\
    & basis\_status  & One-hot simplex basis status: basic, upper, lower, zero. \\
    & reduced\_cost  & Normalized reduced cost. \\
    & age            & LP age of the variable. \\
    & sol\_val        & LP solution value of the variable. \\
    &inc\_val /avg\_inc\_val &Value/Average value in the incumbent solutions\\
\hline
\end{tabular}}
\vspace{-3mm}
\end{table*}

\begin{table*}[tb]
\centering

\caption{Relationship between MILP and \lt~ vertex features. The notation $\mathbb{B} \leftrightarrow \mathbb{Z}$ denotes the potential mutual conversion between binary variables and integer variables.}
\vspace{-3mm}
\label{tab:ltmilp}
\begin{tabular}{p{3.2cm}<{\centering}p{2.2cm}<{\centering}p{4.5cm}<{\centering}}
\hline
\textbf{Vertex feature} & \textbf{MILP} & \lt~ \\
\hline
\multicolumn{3}{c}{\textbf{Constraint vertex features (constraint $i$)}} \\
\hline
obj\_cos\_sim & $C_{i,1}$ & $C_{i,1}$ \\
bias & $C_{i,2}$ & $C_{i,2} + a_i^{\top} t/||a_i||$ \\
is\_tight & $C_{i,3}$ & $C_{i,3}$ \\
dualsol\_val & $C_{i,4}$ & $C_{i,4}$ \\
age & $C_{i,5}$ & $C_{i,5}$ \\
\hline
\multicolumn{3}{c}{\textbf{Edge features (edge $(i, j)$ of constraint $i$ and variable $j$)}} \\
\hline
coef&$E_{i,j}$&$T_{jj}E_{i,j}$\\
\hline
\multicolumn{3}{c}{\textbf{Variable vertex features (variable $j$)}} \\
\hline 
type & $V_{j,1}$ & $V_{j,1}$ or $\mathbb{B} \leftrightarrow \mathbb{Z}$ \\
coef & $V_{j,2}$ & $T_{jj}V_{j,2}$ \\
has\_lb & $V_{j,3}$ & if $T_{jj}=1$, $V_{j,3}$, else $V_{j,4}$ \\
has\_ub & $V_{j,4}$ & if $T_{jj}=1$, $V_{j,4}$, else $V_{j,3}$ \\
sol\_is\_at\_lb & $V_{j,5}$ & if $T_{jj}=1$, $V_{j,5}$, else $V_{j,6}$ \\
sol\_is\_at\_ub & $V_{j,6}$ & if $T_{jj}=1$, $V_{j,6}$, else $V_{j,5}$ \\
sol\_frac & $V_{j,7}$ & if $T_{jj}=1$, $V_{j,7}$, else 1 - $V_{j,7}$ (\texttt{mod} 1) \\
basis\_status & $V_{j,8}$ & $V_{j,8}$ \\
reduced\_cost & $V_{j,9}$ & $T_{jj}V_{j,9}$ \\
age & $V_{j,10}$ & $V_{j,10}$ \\
sol\_val & $V_{j,11}$ & $T_{jj}V_{j,11} + t_j$ \\
inc\_val & $V_{j,12}$ & $T_{jj}V_{j,12} + t_j$ \\
avg\_inc\_val & $V_{j,13}$ & $T_{jj}V_{j,13} + t_j$ \\
\hline
\end{tabular}
\vspace{-2mm}
\end{table*}

\subsection{Graph Features of \lt~  }
The relationships between MILP and \lt~ vertex features are illustrated in Table~\ref{tab:ltmilp}.

Specifically, consider the linear transformation
\[
\setlength\abovedisplayskip{2pt}
\setlength\belowdisplayskip{2pt}
\hat{\mathbf{x}} = \phi(\mathbf{x}) = T \mathbf{x} + \mathbf{t},
\]
where \(T\) is a diagonal matrix with entries \(\pm 1\), and \(\mathbf{t}_{\mathcal{I}} \in \mathbb{Z}^k\). 
Let the original LP relaxation of MILP~\eqref{eq:milp} be
\[
\setlength\abovedisplayskip{2pt}
\setlength\belowdisplayskip{2pt}
\min_{\mathbf{x}} \ \mathbf{c}^\top \mathbf{x} \quad \text{s.t.} \quad A \mathbf{x} \le \mathbf{b}, \ \mathbf{l} \le \mathbf{x} \le \mathbf{u}.
\]
After the transformation, the LP becomes
\[
\setlength\abovedisplayskip{2pt}
\setlength\belowdisplayskip{2pt}
\min_{\hat{\mathbf{x}}} \ \hat{\mathbf{c}}^\top \hat{\mathbf{x}} \quad \text{s.t.} \quad \hat{A} \hat{\mathbf{x}} \le \hat{\mathbf{b}}, \ \hat{\mathbf{l}} \le \hat{\mathbf{x}} \le \hat{\mathbf{u}},
\]
with
\(\hat{\mathbf{c}} = T^\top \mathbf{c}\), \(\hat{A} = A T\), \(\hat{\mathbf{b}} = \mathbf{b} + A T\mathbf{t}\), \(\hat{\mathbf{l}} = T \mathbf{l} + \mathbf{t}\), \(\hat{\mathbf{u}} = T \mathbf{u} + \mathbf{t}\).

\stitle{Dual Solution Value (\texttt{dualsol\_val}).} Let \(\boldsymbol{y}\) denote the dual variables corresponding to \(A \mathbf{x} \le \mathbf{b}\). 
The dual LP is
\[
\max_{\boldsymbol{y} \ge 0} \ \mathbf{b}^\top \boldsymbol{y} \quad \text{s.t.} \quad A^\top \boldsymbol{y} + \mathbf{s} = \mathbf{c}, \ \mathbf{s} \ge 0,
\]
where \(\mathbf{s}\) are slack variables for bounds.  

After transformation, the dual variables \(\hat{\boldsymbol{y}}\) satisfy
\[
\hat{A}^\top \hat{\boldsymbol{y}} + \hat{\mathbf{s}} = \hat{\mathbf{c}}
\quad \Rightarrow \quad (A T)^\top \hat{\boldsymbol{y}} + \hat{\mathbf{s}} = T^\top \mathbf{c}.
\]

Since \(T\) is diagonal with \(\pm 1\), we have
\[
A^\top \hat{\boldsymbol{y}} + T \hat{\mathbf{s}} = \mathbf{c}.
\]
Comparing with the original dual, if we set
\[
\hat{\boldsymbol{y}} = \boldsymbol{y}, \quad \hat{\mathbf{s}} = T \mathbf{s},
\]
the dual feasibility is preserved.  

Thus, the dual solution value for each constraint is preserved.

\stitle{Fractionality of LP Solution (\texttt{sol\_frac}).}
Consider an integer variable \(x_j \in \mathcal{I}\) and its fractionality
\[
\texttt{sol\_frac}_j = \mathrm{frac}(x_j),
\]
where \(\mathrm{frac}(\cdot)\) denotes the fractional part.

After the linear transformation
\[
\hat{x}_j = T_{jj} x_j + t_j, \quad T_{jj} \in \{+1,-1\}, \ t_j \in \mathbb{Z},
\]
the fractionality becomes
\[
\hat{\texttt{sol\_frac}}_j = \mathrm{frac}(\hat{x}_j) = \mathrm{frac}(T_{jj} x_j + t_j).
\]

Since adding an integer \(t_j\) does not change the fractional part, we have
\[
\hat{\texttt{sol\_frac}}_j = \mathrm{frac}(T_{jj} x_j).
\]

Now consider the two cases for \(T_{jj}\):

\begin{itemize}
    \item \(T_{jj} = +1\): \(\hat{\texttt{sol\_frac}}_j = \mathrm{frac}(x_j) = \texttt{sol\_frac}_j\). The fractionality remains unchanged.
    \item \(T_{jj} = -1\): \(\hat{\texttt{sol\_frac}}_j = \mathrm{frac}(-x_j)\). Recall that
    \[
    \mathrm{frac}(-x_j) =
    \begin{cases}
    0, & \text{if } \mathrm{frac}(x_j) = 0,\\
    1 - \mathrm{frac}(x_j), & \text{otherwise}.
    \end{cases}
    \]
\end{itemize}

In summary, the transformed fractionality can be expressed as
\[
\hat{\texttt{sol\_frac}}_j =
\begin{cases}
\texttt{sol\_frac}_j, & T_{jj} = +1,\\
1 - \texttt{sol\_frac}_j \ (\text{mod } 1), & T_{jj} = -1.
\end{cases}
\]

\stitle{Reduced Cost (\texttt{reduced\_cost}).} Original reduced cost:
\[
r_j = c_j - A_{:,j}^\top \boldsymbol{y}.
\]
where $A_{:j}$ is the $j$-th column of the constraint matrix, and $y$ denotes the dual vector of the constraints.

After transformation,
\[
\hat{r}_j = \hat{c}_j - \hat{A}_{:,j}^\top \hat{\boldsymbol{y}} = T_{jj} c_j - (A T)_{:,j}^\top \boldsymbol{y} = T_{jj} (c_j - A_{:,j}^\top \boldsymbol{y}) = T_{jj} r_j.
\]

Therefore, \(\texttt{reduced\_cost}\) flips sign if \(T_{jj} = -1\) and remains the same if \(T_{jj} = 1\).

\stitle{Bounds-Related Features.} The features \texttt{has\_lb} and \texttt{has\_ub} are binary indicators of whether a variable has finite lower and upper bounds, respectively. 
Similarly, \texttt{sol\_is\_at\_lb} (\texttt{sol\_is\_at\_ub}) indicates whether the LP solution of a variable coincides with its lower (upper) bound. 

Under the transformation \(\hat{x}_j = T_{jj} x_j + t_j\), the new bounds are given by 
\(\hat{l}_j = T_{jj} l_j + t_j\), \(\hat{u}_j = T_{jj} u_j + t_j\). 
Since the existence of finite bounds is preserved by affine transformations, the values of \texttt{has\_lb} and \texttt{has\_ub} remain unchanged. 
However, the identity of active bounds may switch when \(T_{jj}=-1\): if the solution was originally at \(l_j\), after transformation it corresponds to \(\hat{u}_j\), and vice versa for \(u_j\). 

\stitle{Solution-Related Features.}
The feature \texttt{sol\_val} records the LP solution value of a variable, while \texttt{inc\_val} and \texttt{avg\_inc\_val} represent the variable's value in the current incumbent solution and the average value across all incumbent solutions, respectively.  

Under the linear transformation \(\hat{x}_j = T_{jj} x_j + t_j\), these values transform consistently as
\[
\setlength\abovedisplayskip{2pt}
\setlength\belowdisplayskip{2pt}
\hat{\texttt{sol\_val}}_j = T_{jj} \cdot \texttt{sol\_val}_j + t_j, \quad 
\hat{\texttt{inc\_val}}_j = T_{jj} \cdot \texttt{inc\_val}_j + t_j, \]
\[
\setlength\abovedisplayskip{2pt}
\setlength\belowdisplayskip{2pt}
\hat{\texttt{avg\_inc\_val}}_j = T_{jj} \cdot \texttt{avg\_inc\_val}_j + t_j.
\]

Therefore, these features undergo the same affine transformation as the variables themselves.

\vspace{-2mm}
\subsection{Graph Features of  \rc~ }
After augmenting the MILP with a redundant constraint $a_r^\top x \le b_r$, a new constraint vertex is added to $\mathcal{G}$ while other vertices' features remain unchanged. 

\stitle{Constraint Vertex features.}
For the newly added vertex, its features such as \texttt{obj\_cos\_sim} and \texttt{bias} are computed from $a_r$ and $b_r$, while \texttt{is\_tight} and \texttt{dualsol\_val} are set to zero, since the redundant constraint does not affect the optimal solution boundary and re-solving the LP will not make it tight. \texttt{age} is also set to zero.

\stitle{Edge features.}
The newly added vertex is connected to variable vertices with nonzero coefficients in $a_r$, with \texttt{coef} reflecting the normalized values.



\end{document}